\documentclass[10pt,journal,compsoc]{IEEEtran}

\ifCLASSOPTIONcompsoc
  \usepackage[nocompress]{cite}
\else
  \usepackage{cite}
\fi

\usepackage{amsfonts}
\usepackage{balance}
\usepackage{booktabs}
\usepackage{amsmath}
\usepackage{tabularx}
\usepackage{subfigure}
\usepackage{graphicx}
\usepackage{multirow}
\usepackage{rotating}
\usepackage{algorithmic}

\usepackage{amsthm}
\newtheorem{mydef}{Definition}
\newtheorem{mythm}{Theorem}

\begin{document}
%
\title{Robust Ensemble Clustering Using\\ Probability Trajectories}

\author{Dong~Huang,~\IEEEmembership{Member,~IEEE,}
        Jian-Huang~Lai,~\IEEEmembership{Senior Member,~IEEE,}
        and~Chang-Dong~Wang,~\IEEEmembership{Member,~IEEE}
\IEEEcompsocitemizethanks{\IEEEcompsocthanksitem Dong Huang is with the School of Information Science and Technology,
Sun Yat-sen University, Guangzhou, China, and with the College of Mathematics and Informatics, South China Agricultural University, Guangzhou, China. \protect\\
E-mail:huangdonghere@gmail.com.}
\IEEEcompsocitemizethanks{\IEEEcompsocthanksitem Jian-Huang Lai is
with the School of Information Science and Technology,
Sun Yat-sen University, Guangzhou, China and with Guangdong Key Laboratory of Information Security Technology, Guangzhou, China. \protect\\
E-mail:stsljh@mail.sysu.edu.cn.}
\IEEEcompsocitemizethanks{\IEEEcompsocthanksitem Chang-Dong Wang is
with the School of Mobile Information Engineering,
Sun Yat-sen University, Zhuhai, China, and with SYSU-CMU Shunde International Joint Research Institute (JRI), Shunde, China. \protect\\
E-mail:changdongwang@hotmail.com.}
\thanks{The MATLAB code and experimental data of this work are available at:   {https://www.researchgate.net/publication/284259332}}}

\IEEEcompsoctitleabstractindextext{%
\begin{abstract}
Although many successful ensemble clustering approaches have been developed in recent years, there are still two limitations to most of the existing approaches. First, they mostly overlook the issue of uncertain links, which may mislead the overall consensus process. Second, they generally lack the ability to incorporate global information to refine the local links. To address these two limitations, in this paper, we propose a novel ensemble clustering approach based on sparse graph representation and probability trajectory analysis. In particular, we present the elite neighbor selection strategy to identify the uncertain links by locally adaptive thresholds and build a sparse graph with a small number of probably reliable links. We argue that a small number of probably reliable links can lead to significantly better consensus results than using all graph links regardless of their reliability. The random walk process driven by a new transition probability matrix is utilized to explore the global information in the graph. We derive a novel and dense similarity measure from the sparse graph by analyzing the probability trajectories of the random walkers, based on which two consensus functions are further proposed. Experimental results on multiple real-world datasets demonstrate the effectiveness and efficiency of our approach.
\end{abstract}

\begin{IEEEkeywords}
Ensemble clustering, consensus clustering, uncertain links, random walk, probability trajectory
\end{IEEEkeywords}}

\maketitle

\IEEEdisplaynotcompsoctitleabstractindextext

\IEEEpeerreviewmaketitle

\IEEEraisesectionheading{\section{Introduction}}

\IEEEPARstart{T}{he} ensemble clustering technique has recently been drawing increasing attention due to its ability to combine multiple clusterings to achieve a probably better and more robust clustering \cite{Fred05_EAC,LiT08,yang11_tkde,iam_on11_linkbased,LiT11,iamon12_tkde,yi_icdm12,ren13_icdm}. Despite the fact that many ensemble clustering approaches have been developed in recent years, there are still two limitations to most of the existing approaches.

First, the existing approaches mostly overlook the problem of uncertain links (or unreliable links) in the ensemble, which may mislead the overall consensus process. In the general formulation of the ensemble clustering problem \cite{Fred05_EAC,iam_on11_linkbased,iamon12_tkde,yi_icdm12}, we have no access to the original data features, as only the different types of relational information are available. The most basic relational information are the links between objects which reflect how objects are grouped in the same cluster or different clusters in the ensemble \cite{Fred05_EAC}. Based on the object-object links, coarser grains of links can be defined, e.g., the links between objects and clusters, the links between clusters, the links between base clusterings, etc. The links of one type or different types can be further used to form the similarity matrix \cite{Fred05_EAC,yang11_tkde,iam_on11_linkbased}, construct the graph model \cite{strehl02,fern04_bipartite,Mimaroglu11_pr}, or define the optimization problem \cite{LiT08,cristofor02,topchy05}.

A link between two data objects denotes that they appear in the same cluster in one or more base clusterings. The links between objects are typically represented by the co-association (CA) matrix \cite{Fred05_EAC,yang11_tkde,wang09_pr}. In the similarity graph induced by the CA matrix, each node corresponds to a data object and the weight of each link corresponds to an entry in the CA matrix. However, previous approaches generally overlook the different reliability of the links (or the entries in the CA matrix) and may suffer from the collective influence of the unreliable links. As an example, we construct an ensemble of $10$ base clusterings for the $\emph{MNIST}$ dataset (see Section~\ref{sec:construct_base} for more details). Figure~\ref{fig:link_weight_distribution} illustrates the distribution of the link weights in the similarity graph induced by the CA matrix, and Table~\ref{table:link_weight_distribution} shows the percentages of the links with different weights that make \emph{correct} decisions. Note that a "link" with zero-weight does not count as a link here. If there is a link between objects $x_i$ and $x_j$ in the graph AND $x_i$ and $x_j$ are in the same class in the ground-truth, then we say the link between $x_i$ and $x_j$ makes a \emph{correct} decision. As shown in Table~\ref{table:link_weight_distribution}, the links with greater weights are much more likely to make correct decisions and generally more reliable than the small-weight links. But unfortunately, the small-weight links, which are probably unreliable (or uncertain), make up the majority of the graph links (see Fig.~\ref{fig:link_weight_distribution}). When the number of the uncertain links is large, the collective influence of them may mislead the ensemble clustering process or even lead to deteriorative clustering results \cite{yi_icdm12}. It remains an open problem how to effectively and efficiently deal with the uncertain links (or probably unreliable links) and thereby enhance the robustness and accuracy of the consensus results.

\begin{figure}[!t]
\begin{center}
{
{\includegraphics[width=0.7\linewidth]{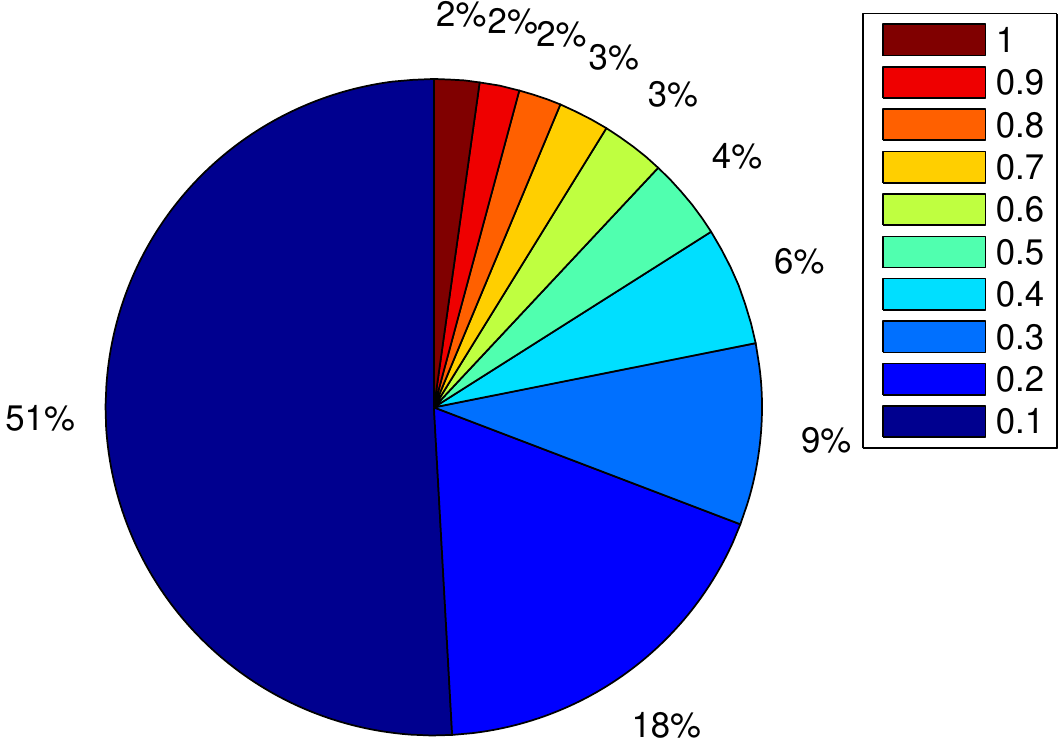}}}
\caption{The distribution of the link weights for the similarity graph induced by the CA matrix for the \emph{MNIST} dataset}
\label{fig:link_weight_distribution}
\end{center}
\end{figure}

\begin{table}[!t]\footnotesize 
\centering 
\caption{The distribution of the links with different weights and the proportion of them making correct decisions for the \emph{MNIST} dataset}
\label{table:link_weight_distribution} 
\begin{tabular}{m{0.7cm}<{\centering}|m{2.4cm}<{\centering}|m{4.6cm}<{\centering}}
\toprule
Weight  &$\frac{\#\textrm{Links with the weight}}{\#\textrm{All (non-zero) Links}}$    &$\frac{\#\textrm{Links with the weight that make correct decisions}}{\#\textrm{Links with the weight}}$\\
\midrule
$1$  &$2\%$     &$83\%$\\
$0.9$  &$2\%$     &$73\%$\\
$0.8$  &$2\%$     &$65\%$\\
$0.7$  &$3\%$     &$58\%$\\
$0.6$  &$3\%$     &$50\%$\\
$0.5$  &$4\%$     &$43\%$\\
$0.4$  &$6\%$     &$35\%$\\
$0.3$  &$9\%$     &$27\%$\\
$0.2$  &$18\%$     &$18\%$\\
$0.1$  &$51\%$     &$9\%$\\
\bottomrule
\end{tabular}
\end{table}

Second, the existing ensemble clustering approaches mostly lack the ability to utilize global structure information to refine the local links. In the classical evidence accumulation clustering \cite{Fred05_EAC} and some of its extensions \cite{yang11_tkde,wang09_pr}, the CA matrix reflects the local (or \emph{direct}) relationship, i.e., the co-occurrence relationship, between objects, yet generally neglects the \emph{indirect} relationship inherent in the ensemble. Let $x_i$ and $x_j$ be two nodes in the graph. If there is a (non-zero) link between $x_i$ and $x_j$, then we say that there is a \emph{direct} relationship between them. If there exist $\xi$ nodes (with $\xi\geq 1$), say, ${x'}_{1},\cdots,{x'}_{\xi}$, such that there are links between $x_i$ and ${x'}_{1}$, between ${x'}_{\xi}$ and $x_j$, and between ${x'}_{k}$ and ${x'}_{k+1}$ for any $1\leq k <\xi$, then $x_i$ and $x_j$ are indirectly connected by the $\xi$ nodes. Here we say there is a $\xi$-step \emph{indirect} relationship between $x_i$ and $x_j$. We refer to the entirety of the direct relationships and the different steps of indirect relationships as the \emph{global structure information}. Three sample graphs are illustrated in Fig.~\ref{fig:global_local1}, Fig.~\ref{fig:global_local2}, and Fig.~\ref{fig:global_local3}, respectively. Although the direct link weights between $x_1$ and $x_2$ and between $x_2$ and $x_3$ remain unchanged across the three sample graphs in Fig.~\ref{fig:global_locals}, the global structures of the three graphs are very different. The local links may be affected by the noise and the inherent complexity of the real-world datasets, while the global structure information is more robust to the potential local errors and can provide an alternative way to explore the relationship between objects.  The relationship between objects lies not only in the direct connections, but also in the indirect connections \cite{iam_on11_linkbased}. The key problem here is how to exploit the global structure information in the ensemble effectively and efficiently and thereby improve the final clustering results.

Aiming to tackle the aforementioned two limitations, in this paper, we propose an ensemble clustering approach based on sparse graph representation and probability trajectory analysis.

We introduce the concept of microclusters as a compact representation for the ensemble data, which is able to greatly reduce the problem size and  facilitate the computation. To deal with the uncertain links, we propose a $k$-nearest-neighbor-like strategy termed elite neighbor selection (ENS) to identify the uncertain links and build a sparse graph termed $K$-elite neighbor graph ($K$-ENG) that preserves only a small number of probably reliable links. Two microclusters are elite neighbors if the link between them has a large weight (see Definition~\ref{def:K_EN}). We argue that the use of a small number of probably reliable links is capable of reflecting the overall structure of the graph and may lead to much better and more robust clustering results than using all graph links without considering their reliability. According to our experimental analysis, we find that only preserving several percent or even less than one percent of the links via the ENS strategy can result in significant improvements to the final consensus results compared to using all links in the original graph.

\begin{figure}[!t]
\begin{center}
{\subfigure[]
{\includegraphics[width=0.326\linewidth]{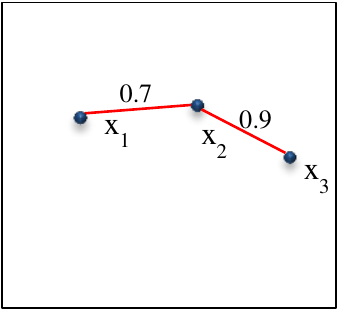}\label{fig:global_local1}}}
{\subfigure[]
{\includegraphics[width=0.326\linewidth]{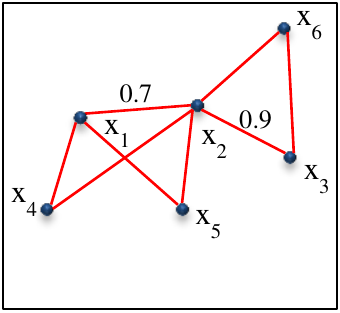}\label{fig:global_local2}}}
{\subfigure[]
{\includegraphics[width=0.326\linewidth]{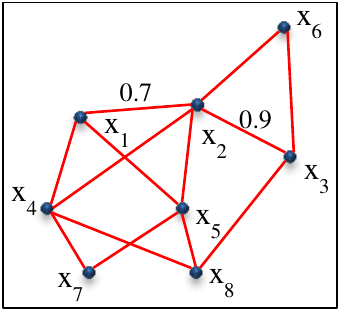}\label{fig:global_local3}}}
\caption{How are the relationships between $x_1$ and $x_2$ and between $x_2$ and $x_3$ in the context of (a) a graph with three nodes, (b) a graph with six nodes, and (c) a graph with eight nodes?}
\label{fig:global_locals}
\end{center}
\end{figure}

Having constructed the $K$-ENG graph, we proceed to exploit the global structure information to refine local links by means of random walks. The random walk technique has proved to be a powerful tool for finding community structures (or cluster structures) in the field of community discovery \cite{newman04,pons_rw_05,lai_PRE10,neiwalk14_tkde}. The random walks are performed on the sparse graph $K$-ENG and each node in $K$-ENG is treated as a start node for a random walker. We propose a new transition probability matrix that simultaneously considers the link weights and the node sizes. By analyzing the probability trajectories of the random walkers starting from different initial nodes, we derive a novel and dense similarity measure termed probability trajectory based similarity (PTS) from the sparse graph. The PTS explores the pair-wise relationships by capturing the global structure information in the sparse graph via random walk trajectories. Based on PTS, two consensus functions are further proposed, namely, probability trajectory accumulation (PTA) and probability trajectory based graph partitioning (PTGP).

The overall process of our approach is illustrated in Fig.~\ref{fig:flowchart}. Given the ensemble of base clusterings, we first map the data objects to a set of microclusters and compute the microcluster based co-association (MCA) matrix. With each microcluster treated as a node, we construct the microcluster similarity graph (MSG) according to the MCA matrix. Then, the ENS strategy is performed on the MSG and the sparse graph $K$-ENG is constructed by preserving a small number of probably reliable links. The random walks are conducted on the $K$-ENG graph and the PTS similarity is obtained by comparing random walk trajectories. Having computed the new similarity matrix, any pair-wise similarity based clustering methods can be used to achieve the consensus clustering. Typically, we propose two novel consensus functions, termed PTA and PTGP, respectively. Note that PTA is based on agglomerative clustering, while PTGP is based on graph partitioning. Figure~\ref{fig:ensize9in1} summarizes the average performances (over nine real-world datasets) of the proposed PTA and PTGP methods and four existing methods, i.e., MCLA \cite{strehl02}, GP-MGLA \cite{huang14_weac}, EAC \cite{Fred05_EAC}, and WEAC \cite{huang14_weac}, with varying ensemble sizes. Each method is run 20 times on each dataset and their average NMI scores are shown in Fig.~\ref{fig:ensize9in1}. By dealing with the uncertain links and the global structure information, the proposed PTA and PTGP methods exhibit a significant advantage in clustering robustness (to various datasets and ensemble sizes) over the baseline methods. Please see Section~\ref{sec:experiment} for more extensive details of our experimental evaluation.

\begin{figure}[!t]
\begin{center}
{
{\includegraphics[width=0.67\linewidth]{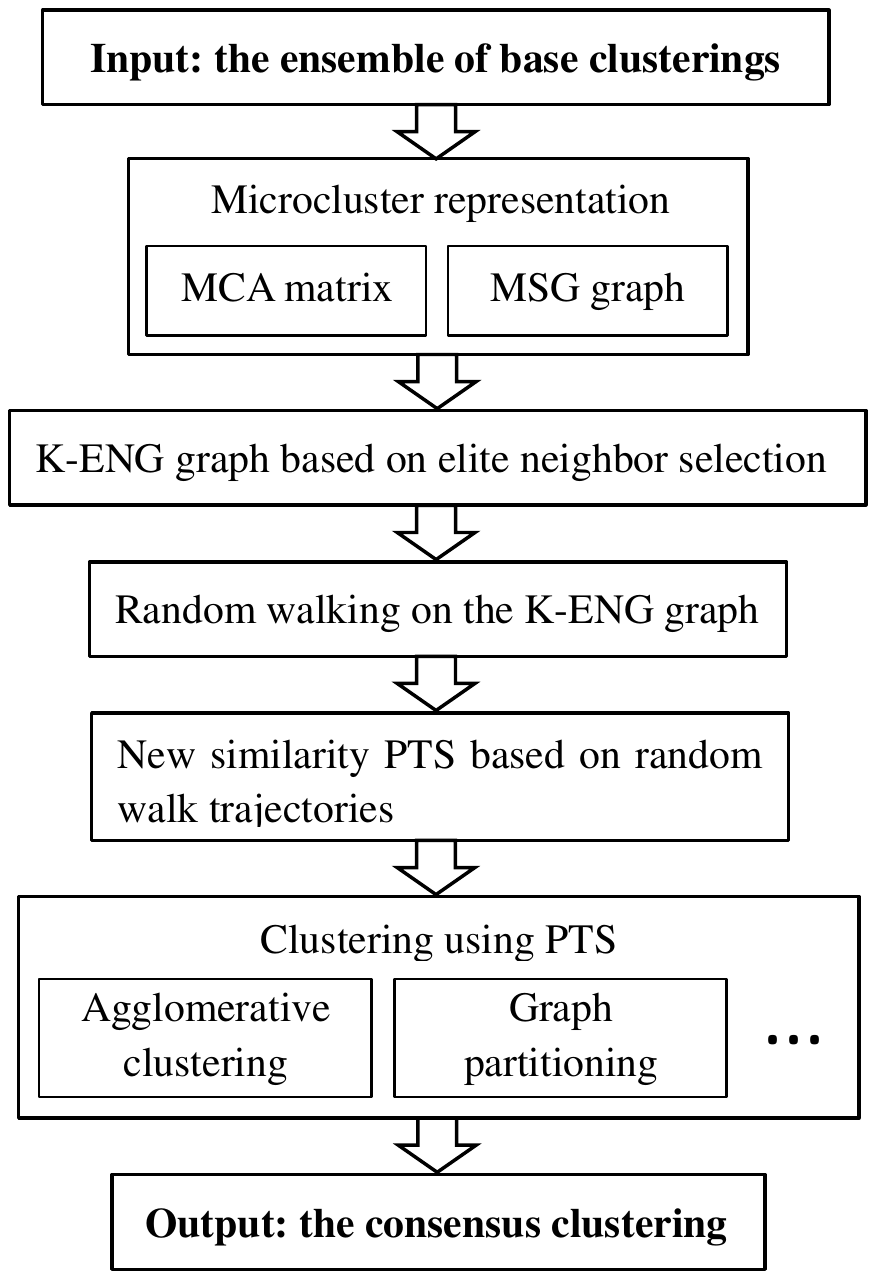}}}
\caption{Flow diagram of the proposed approach.}\label{fig:flowchart}
\end{center}
\end{figure}

The main contributions of our approach are summarized as follows:

\begin{enumerate}
  \item Our approach addresses the issue of uncertain links in an effective and efficient manner. We propose to identify the uncertain links by the ENS strategy and build a sparse graph with a small number of probably reliable links. Our empirical study shows the advantage of using only a small number of probably reliable links rather than all graph links regardless of their reliability.
  \item Our approach is able to incorporate global information to construct more accurate local links by exploiting the random walk trajectories. The random walkers driven by a new probability transition matrix are utilized to explore the graph structure. A dense similarity measure is further derived from the sparse graph $K$-ENG using probability trajectories of the random walkers.
  \item Extensive experiments are conducted on a variety of real-world datasets. The experimental results show that our approach yields significantly better performance than the state-of-the-art approaches w.r.t. both clustering accuracy and efficiency.
\end{enumerate}

\begin{figure}[!t]
\begin{center}
{
{\includegraphics[width=0.65\linewidth]{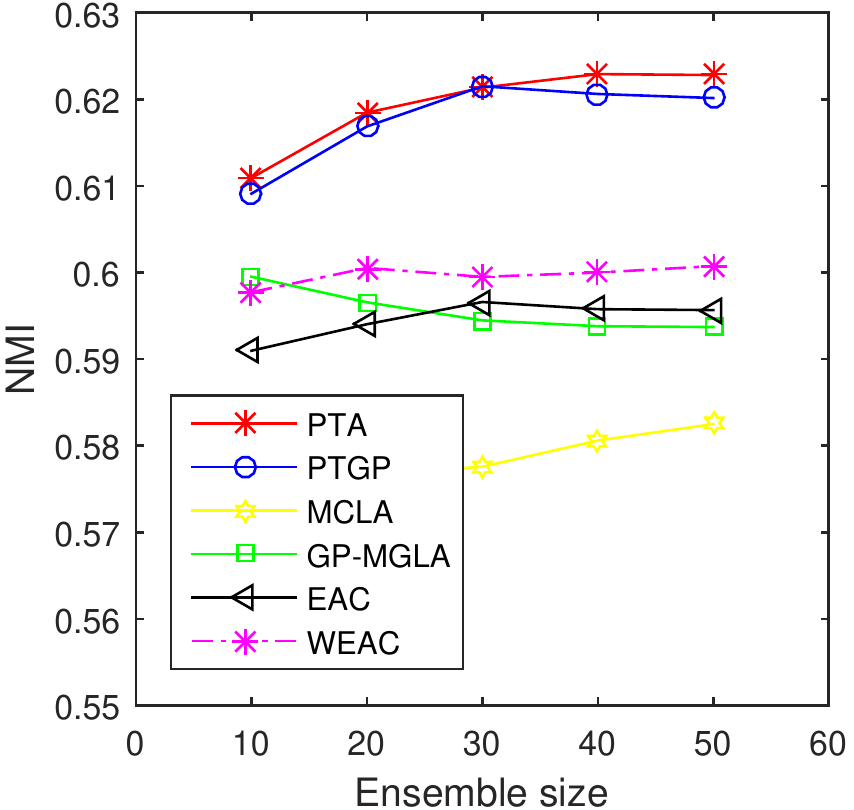}}}
\caption{Average performances (in terms of NMI) of different approaches over nine different real-world datasets with varying ensemble size $M$.}\label{fig:ensize9in1}
\end{center}
\end{figure}

The rest of this paper is organized as follows. The related work is introduced in Section~\ref{sec:related_work}. The formulation of the ensemble clustering problem is provided in Section~\ref{sec:formulation}. The proposed ensemble clustering approach is described in Section~\ref{sec:our_framework}. The experimental results are reported in Section~\ref{sec:experiment}. Section~\ref{sec:conclusion} concludes the paper.

\section{related work}
\label{sec:related_work}

There is a large amount of literature on ensemble clustering in the past few decades \cite{vega_pons11_survey}. The existing ensemble clustering approaches can be categorized into three main classes, that is, (i) the pair-wise similarity based approaches \cite{Fred05_EAC,yang11_tkde,wang09_pr,li07}, (ii) the median partition based approaches \cite{cristofor02,topchy05}, and (iii) the graph partitioning based approaches \cite{strehl02,fern04_bipartite,abdala10_icpr,ren13_icdm}.

The pair-wise similarity based approaches represent the ensemble information by some pair-wise similarity measure \cite{Fred05_EAC,yang11_tkde,wang09_pr,li07}. The evidence accumulation clustering (EAC) proposed by Fred and Jain \cite{Fred05_EAC} is probably the best-known pair-wise similarity based approach. In EAC, a co-association (CA) matrix is constructed by counting how many times two objects occur in the same cluster in the ensemble of multiple base clusterings. By treating the CA matrix as a new similarity matrix, clustering algorithms, such as the agglomerative clustering methods, can be further utilized to obtain the consensus clustering. Li et al. \cite{li07} proposed a hierarchical clustering algorithm to construct the consensus clustering using the CA matrix. The concept of normalized edges is introduced in \cite{li07} to measure the similarity between clusters. Wang et al. \cite{wang09_pr} extended the EAC method by taking the cluster sizes into consideration and proposed the probability accumulation method.

The median partition based approaches aim to find a clustering (or partition) that maximizes the similarity between this clustering and all of the base clusterings, which can be viewed as finding the median point of the base clusterings \cite{cristofor02,topchy05}. Due to the huge space of all possible clusterings, it is generally infeasible to find the optimal solution for the median partition problem. In fact, the median partition problem is NP-complete \cite{topchy05}. Cristofor and Simovici \cite{cristofor02} proposed to find an approximative solution for the ensemble clustering problem by exploiting the genetic algorithm, in which clusterings are represented as chromosomes. Topchy et al. \cite{topchy05} proposed to formulate the median partition problem into a maximum likelihood problem and solved it by the EM algorithm.

The graph partitioning based approaches are another main category of ensemble clustering \cite{ren13_icdm,strehl02,fern04_bipartite}. Strehl and Ghosh \cite{strehl02} formulated the ensemble clustering problem into a graph partitioning problem and proposed three ensemble clustering approaches: cluster-based similarity partitioning algorithm (CSPA), hypergraph partitioning algorithm (HGPA), and meta-clustering algorithm (MCLA). Fern and Brodley \cite{fern04_bipartite} formulated the clustering ensemble into a bipartite graph by treating both clusters and objects as graph nodes and obtained the consensus clustering by partitioning the bipartite graph. Ren et al. \cite{ren13_icdm} proposed to assign weights to data objects with regard to how difficult it is to cluster them and presented three graph partitioning algorithms based on the weighted-object scheme, that is, weighted-object meta clustering (WOMC), weighted-object similarity partition (WOSP) clustering, and weighted-object hybrid bipartite (WOHB) graph partition clustering.

Despite the significant success, there are still two limitations to most of the existing ensemble clustering approaches. First, the existing approaches mostly overlook the problem of uncertain links which may mislead the consensus process. Second, most of them lack the ability to incorporate global structure information to refine local links accurately and efficiently. Recently, some efforts have been made to address these two limitations. To deal with the uncertain links, Yi et al. \cite{yi_icdm12} proposed an ensemble clustering approach based on global thresholding and matrix completion. However, using global thresholds may lead to \emph{isolated} nodes, i.e., all of the links connected to a node may be cut out, due to the lack of local adaptivity. Moreover, in the approach of \cite{yi_icdm12}, a parameter $C$ is used to scale the noise term in the objective function (see \cite{yi_icdm12} for more details) and plays a sensitive and crucial role for yielding a good consensus clustering. Without knowing the ground-truth in advance, tuning the parameter $C$ for \cite{yi_icdm12} is very difficult and computationally expensive. To exploit the indirect relationships in the ensemble, Iam-On et al. \cite{iam_on11_linkbased,iamon12_tkde} proposed to refine the CA matrix by considering the shared neighbors between clusters. The approaches in \cite{iam_on11_linkbased} and \cite{iamon12_tkde} utilize the common neighborhood information, i.e., the $1$-step indirect relationships, and have not gone beyond the $1$-step indirect relationships to explore the more comprehensive structure information in the ensemble. To utilize multi-step indirect structure information, in the work of \cite{iamon08_icds}, Iam-on et al. proposed to refine pair-wise links by the SimRank similarity (SRS), which, however, suffers from its high computational complexity and is not feasible for large datasets (see Fig.~\ref{fig:time_complexity}). Different from \cite{yi_icdm12} and \cite{iamon08_icds}, in this paper, we propose to tackle these two limitations in a unified and efficient manner. We present the elite neighbor selection strategy to identify the uncertain links by locally adaptive thresholds and build a sparse graph with a small number of probably reliable links, which has shown its advantage compared to using the whole original graph without considering the reliability of the links (see Section~\ref{sec:experiment}). To explore the global structure information of the graph, we exploit the random walk process with a new transition probability matrix. By analyzing the probability trajectories of random walkers, a novel and dense similarity measure termed PTS is derived from the sparse graph. Specifically, based on the PTS, two ensemble clustering algorithms are further proposed, termed PTA and PTGP, respectively.

\section{Problem Formulation}
\label{sec:formulation}

\begin{table}[!t]
\caption{Summary of notations}\vskip -0.2 in
\label{table:notations}
\begin{center}
\begin{tabular}{c|l}
\toprule
$N$         &Number of data objects in a dataset\\
$x_i$       &Data object\\
$\mathcal{X}$   &Dataset of $N$ objects, $\mathcal{X}=\{x_1,\cdots,x_N\}$\\
$n^k$         &Number of clusters in the $k$-th base clustering\\
$C^k_j$         &The $j$-th cluster in the $k$-th base clustering\\
$\pi^k$         &The $k$-th base clustering, $\pi^k=\{C^k_1,\cdots,C^k_{n^k}\}$\\
$M$         &Number of base clusterings\\
$\Pi$         &Ensemble of $M$ base clusterings, $\Pi=\{\pi^1,\cdots,\pi^M\}$\\
$\pi^*$         &Consensus clustering\\
$N_c$       &Total number of clusters in $\Pi$\\
$C_i$       &The $i$-th cluster in $\Pi$\\
$\mathcal{C}$   &Set of clusters in $\Pi$, $\mathcal{C}=\{C_1,\cdots,C_{N_c}\}$\\
$\tilde{N}$         &Number of microclusters\\
$y_i$         &The $i$-th microcluster\\
$\mathcal{Y}$         &Set of microclusters, $\mathcal{Y}=\{y_1,\cdots,y_{\tilde{N}}\}$\\
$\tilde{n}_i$       &Number of data objects in $y_i$\\
$b_{ij}$              &Number of times $x_i$ and $x_j$ occur in the same cluster in $\Pi$\\
$a_{ij}$              &Entry of co-association (CA) matrix\\
$A$              &CA matrix, $A=\{a_{ij}\}_{N\times N}$\\
$\tilde{b}_{ij}$              &Number of times $y_i$ and $y_j$ occur in the same cluster in $\Pi$\\
$\tilde{a}_{ij}$              &Entry of microcluster based co-association (MCA) matrix\\
$\tilde{A}$              &MCA matrix, $\tilde{A}=\{\tilde{a}_{ij}\}_{\tilde{N}\times \tilde{N}}$\\
$\tilde{G}$              &Microcluster similarity graph (MSG)\\
$\tilde{V}$              &Node set of $\tilde{G}$. Note that $\tilde{V}=\mathcal{Y}$\\
$\tilde{L}$              &Link set of $\tilde{G}$\\
$\tilde{w}_{ij}$         &Weight between two nodes in $\tilde{G}$\\
$\bar{G}$              &$K$-elite neighbor graph ($K$-ENG)\\
$\bar{V}$              &Node set of $\bar{G}$. Note that $\tilde{V}=\mathcal{Y}$\\
$\bar{L}$              &Link set of $\bar{G}$\\
$\bar{w}_{ij}$              &Weight between two nodes in $\bar{G}$\\
$p_{ij}$              &(1-step) transition probability from $y_i$ to $y_j$\\
$P$                   &(1-step) transition probability matrix, $P=\{p_{ij}\}_{\tilde{N}\times\tilde{N}}$\\
$p^T_{ij}$              &$T$-step transition probability from $y_i$ to $y_j$\\
$P^T$                   &$T$-step transition probability matrix, $P^T=\{p^T_{ij}\}_{\tilde{N}\times\tilde{N}}$\\
$p^T_{i:}$              &The $i$-th row of $P^T$, $p^T_{i:}=\{p^T_{i1},\cdots,p^T_{i\tilde{N}}\}$\\
$PT^T_i$            &Probability trajectory of a random walker starting from\\
                    &node $y_i$ with length $T$\\
$PTS_{ij}$      &Probability trajectory based similarity between $y_i$ and $y_j$\\
$\mathcal{R}^{(0)}$ &Set of the initial regions for PTA,\\ &$\mathcal{R}^{(0)}=\{R^{(0)}_1,\cdots,R^{(0)}_{|\mathcal{R}^{(0)}|}\}$\\
$S^{(0)}$           &Initial similarity matrix for PTA,\\ &$S^{(0)}=\{s^{(0)}_{ij}\}_{|\mathcal{R}^{(0)}|\times |\mathcal{R}^{(0)}|}$\\
$\mathcal{R}^{(t)}$ &Set of the $t$-step regions for PTA,\\ &$\mathcal{R}^{(t)}=\{R^{(t)}_1,\cdots,R^{(t)}_{|\mathcal{R}^{(t)}|}\}$\\
$S^{(t)}$           &The $t$-step similarity matrix for PTA,\\ &$S^{(t)}=\{s^{(t)}_{ij}\}_{|\mathcal{R}^{(t)}|\times |\mathcal{R}^{(t)}|}$\\
$\ddot{G}$       &Microcluster-cluster bipartite graph (MCBG)\\
$\ddot{N}$       &Number of nodes in $\ddot{G}$\\
$\ddot{V}$       &Node set of $\ddot{G}$\\
$\ddot{L}$       &Link set of $\ddot{G}$\\
$\ddot{w}_{ij}$  &Weight between two nodes in $\ddot{G}$\\
\bottomrule
\end{tabular}
\end{center}
\end{table}

In Table~\ref{table:notations}, we summarize the notations that are used throughout the paper. Let $\mathcal{X}=\{x_1,\cdots,x_N\}$ be a dataset of $N$ objects. Given $M$ partitions of $\mathcal{X}$, each treated as a base clustering, the goal is to find a consensus clustering $\pi^*$ that summarizes the information of the ensemble. The ensemble is denoted as $\Pi=\{\pi^1,\cdots,\pi^M\}$, where $\pi^k=\{C^k_1,\cdots,C^k_{n^k}\}$ is the $k$-th base clustering. Let $Cls^k(x_i)$ be the cluster in $\pi^k$ that contains object $x_i$. If $x_i\in C^k_j$, then $Cls^k(x_i)=C^k_j$. Let $\mathcal{C}=\{C_1,\cdots,C_{N_c}\}$ be the set of clusters in all of the $M$ base clusterings. Obviously, $N_c=\sum_{k=1}^M n^k$.

With regard to the difference in the input information, there are two formulations of the ensemble clustering problem. In the first formulation, the ensemble clustering system takes both the clustering ensemble $\Pi$ and the data features of $\mathcal{X}$ as inputs \cite{vega_pons10,vega_pons_PRL11}. In the other formulation, the ensemble clustering system takes only the clustering ensemble $\Pi$ as input and has no access to the original feature vectors of the dataset \cite{Fred05_EAC,iam_on11_linkbased,iamon12_tkde,yi_icdm12}. In this paper, we comply with the latter formulation of the ensemble clustering problem, which is also the common practice for most of the existing ensemble clustering algorithms \cite{vega_pons11_survey}. That is, in our formulation, the \emph{input} of the ensemble clustering system is the ensemble $\Pi$, and the \emph{output} is the consensus clustering $\pi^*$.

\section{Ensemble Clustering Using Probability Trajectories}
\label{sec:our_framework}

In this section, we describe the proposed ensemble clustering approach based on sparse graph representation and probability trajectory analysis.

The overall process of the proposed approach is illustrated in Fig.~\ref{fig:flowchart}. The microclusters are used as a compact representation for the clustering ensemble. The microcluster based co-association (MCA) matrix for the ensemble is computed and a microcluster similarity graph (MSG) is constructed from the MCA matrix with the microclusters treated as graph nodes. In order to deal with the uncertain links, we propose a $k$-nearest-neighbor-like pruning strategy termed elite neighbor selection (ENS), which is able to identify the uncertain links by locally adaptive thresholds. A sparse graph termed $K$-elite neighbor graph ($K$-ENG) is then constructed with only a small number of probably reliable links. The ENS strategy is a crucial step in our approach. We argue that using a small number of probably reliable links may lead to significantly better consensus results than using all graph links regardless of their reliability. The random walk process driven by a new transition probability matrix is performed on the $K$-ENG  to explore the global structure information. From the sparse graph $K$-ENG, a dense similarity measure termed PTS is derived by exploiting the probability trajectories of the random walkers. Two novel consensus functions are further proposed, termed PTA and PTGP, respectively. In the following, we describe each step of our framework in detail.

\begin{table}[!t]
\caption{An example of the microcluster representation}\vskip -0.1in
\label{table:mc_example}
\begin{center}
\begin{tabular}{cccc}
\toprule
\multirow{2}{*}{Data objects}             &\multicolumn{2}{c}{Cluster labels}  &\multirow{2}{*}{Microcluster labels}\\
\cline{2-3}
&\begin{minipage}{0.7cm} \vspace{0.1cm}\centering$\pi^1$  \vspace{0.01cm}\end{minipage}
                      &\begin{minipage}{0.7cm} \vspace{0.1cm}\centering$\pi^2$  \vspace{0.01cm}\end{minipage}          \\
\midrule
\begin{minipage}{1.4cm} \vspace{0.15cm}\centering $x_1$ \vspace{0.15cm}\end{minipage}&$C_1^1$&$C_1^2$&$y_1$\\
\begin{minipage}{1.4cm} \vspace{0.15cm}\centering $x_2$ \vspace{0.15cm}\end{minipage}&$C_1^1$&$C_1^2$&$y_1$\\
\begin{minipage}{1.4cm} \vspace{0.15cm}\centering $x_3$ \vspace{0.15cm}\end{minipage}&$C_1^1$&$C_1^2$&$y_1$\\
\begin{minipage}{1.4cm} \vspace{0.15cm}\centering $x_4$ \vspace{0.15cm}\end{minipage}&$C_1^1$&$C_2^2$&$y_2$\\
\begin{minipage}{1.4cm} \vspace{0.15cm}\centering $x_5$ \vspace{0.15cm}\end{minipage}&$C_2^1$&$C_2^2$&$y_3$\\
\begin{minipage}{1.4cm} \vspace{0.15cm}\centering $x_6$ \vspace{0.15cm}\end{minipage}&$C_2^1$&$C_2^2$&$y_3$\\
\begin{minipage}{1.4cm} \vspace{0.15cm}\centering $x_7$ \vspace{0.15cm}\end{minipage}&$C_2^1$&$C_3^2$&$y_4$\\
\begin{minipage}{1.4cm} \vspace{0.15cm}\centering $x_8$ \vspace{0.15cm}\end{minipage}&$C_2^1$&$C_3^2$&$y_4$\\
\bottomrule
\end{tabular}
\end{center}
\end{table}

\begin{figure}[!t]
\begin{center}
{\subfigure[]
{\includegraphics[width=0.37\linewidth]{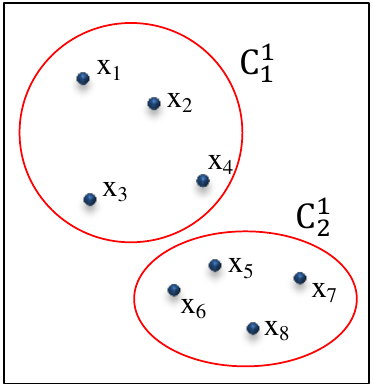}\label{fig:mc_example1}}}
{\subfigure[]
{\includegraphics[width=0.37\linewidth]{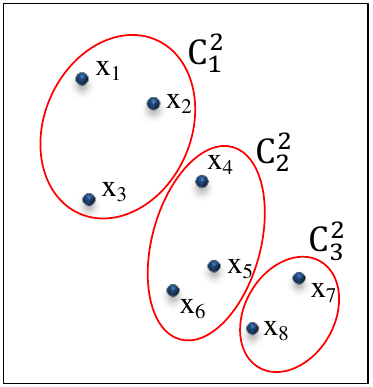}\label{fig:mc_example2}}}
{\subfigure[]
{\includegraphics[width=0.37\linewidth]{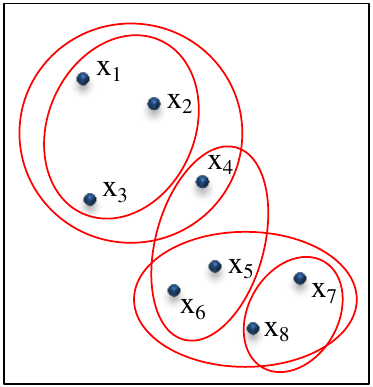}\label{fig:mc_example3}}}
{\subfigure[]
{\includegraphics[width=0.37\linewidth]{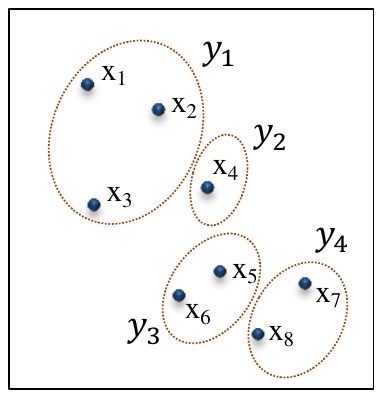}\label{fig:mc_example4}}}
\caption{Illustration of the example in Table~\ref{table:mc_example}. (a) The base clustering $\pi^1$. (b) The base clustering $\pi^2$. (c) The intersection of $\pi^1$ and $\pi^2$. (d) The microcluster representation for the ensemble of $\pi^1$ and $\pi^2$.}
\label{fig:mc_examples}
\end{center}
\end{figure}

\subsection{Microcluster Based Representation} 
\label{sec:MCA}
In this paper, we propose to discover the object relationship by analyzing the probability trajectories of the random walkers. One hurdle in conducting random walks is the computational complexity. In a graph of $N$ nodes, it takes $O(TN^2)$ operations to calculate $T$ steps of random walks, which limits its application in large datasets. A practical strategy is to use a larger granularity than the original objects to reduce the number of graph nodes. In the field of image segmentation, superpixels are often adopted as the primitive objects \cite{CVPR12_Li}. However, with neither the spatial constraints of the image data nor access to the original data features, the conventional superpixel segmentation methods \cite{CVPR12_Li} are not applicable for generating the primitive segments for the ensemble clustering problem. In this paper, we introduce the concept of microclusters as a compact representation for the ensemble. The objects $x_i$ and $x_j$ are defined to be in the same microcluster if and only if they occur in the same cluster for \emph{all} of the $M$ base clusterings, i.e., for $k=1,\cdots,M$, $Cls^k(x_i)=Cls^k(x_j)$.

\begin{mydef}
\label{def:MC}
Let $y$ be a set of objects. The set $y$ is a microcluster if and only if (i) $\forall x_i, x_j\in y$ and $k\in\mathbb{Z}$ s.t. $1\leq k \leq M$, $Cls^k(x_i)=Cls^k(x_j)$, and (ii) $\forall x_i\in y$ and $x_j \not\in y$, $\exists k\in\mathbb{Z}$ s.t. $1\leq k \leq M$, $Cls^k(x_i)\neq Cls^k(x_j)$.
\end{mydef}

Given the clustering ensemble, we can produce a set of $\tilde{N}$ non-overlapping microclusters, denoted as
\begin{equation}
\mathcal{Y}=\{y_1,\cdots,y_{\tilde{N}}\}.
\end{equation}
Intuitively, the set of microclusters is produced by intersecting the $M$ base clusterings. As it is defined, there are no clues to distinguish the objects in the same microcluster given the information of the ensemble $\Pi$. In our work, the microclusters are utilized as the primitive objects. For any microcluster $y_i\in\mathcal{Y}$ and cluster $C_j^k\in \pi^k$, it holds that \emph{either} every object in $y_i$ is in $C_j^k$ \emph{or} no object in $y_i$ is in $C_j^k$. Because the microclusters are treated as primitive objects, in the following, if all objects in $y_i$ are in $C_j^k$, we write it as $y_i\in C_j^k$ rather than $y_i\subseteq C_j^k$; otherwise, we write it as $y_i\not\in C_j^k$ rather than $y_i\not\subseteq C_j^k$. Let $Cls^k(y_i)$ denote the cluster in $\pi^k$ that contains the microcluster $y_i$. If $y_i\in C^k_j$, then $Cls^k(y_i) = C^k_j$.

In Table~\ref{table:mc_example} and Fig.~\ref{fig:mc_examples}, we show an example of a dataset with eight objects to describe the construction of the microclusters. There are two base clusterings in the example, namely, $\pi^1$ and $\pi^2$, which consist of two and three clusters, respectively. With regard to Definition~\ref{def:MC}, four microclusters can be produced using $\pi^1$ and $\pi^2$, which correspond to the intersection of $\pi^1$ and $\pi^2$ (as illustrated in Fig.~\ref{fig:mc_example3} and \ref{fig:mc_example4}).

\begin{mydef}
\label{def:CA}
The co-association (CA) matrix of the clustering ensemble $\Pi$ is defined as
\begin{equation}
A=\{a_{ij}\}_{N\times N},
\end{equation}
where
\begin{equation}
a_{ij}=\frac{b_{ij}}{M}
\end{equation}
and $b_{ij}$ denotes how many times the objects $x_i$ and $x_j$ occur in the same cluster among the $M$ base clusterings.
\end{mydef}

\begin{mydef}
\label{def:MCA}
The microcluster based co-association (MCA) matrix of the clustering ensemble $\Pi$ is defined as
\begin{equation}
\label{eq:MCA}
\tilde{A}=\{\tilde{a}_{ij}\}_{\tilde{N}\times \tilde{N}},
\end{equation}
where
\begin{equation}
\tilde{a}_{ij}=\frac{\tilde{b}_{ij}}{M}
\end{equation}
and $\tilde{b}_{ij}$ denotes how many times the microclusters $y_i$ and $y_j$ occur in the same cluster among the $M$ base clusterings.
\end{mydef}

\begin{mythm}
\label{thm:MCA_CA}
For all $x_i,x_j\in\mathcal{X},y_k,y_l\in\mathcal{Y}$ such that $x_i\in y_k$ and $x_j\in y_l$, it holds that $a_{ij}=\tilde{a}_{kl}$.
\end{mythm}
\begin{proof}
Given $x_i,x_j\in\mathcal{X},y_k,y_l\in\mathcal{Y}$, we have $a_{ij}=b_{ij}/M$ and $\tilde{a}_{kl}=\tilde{b}_{kl}/M$ according to Definitions~\ref{def:CA} and \ref{def:MCA}. To prove $a_{ij}=\tilde{a}_{kl}$, we need to prove $b_{ij}=\tilde{b}_{kl}$. Given $x_i\in y_k$ and $x_j\in y_l$, for any base clustering $\pi^m\in \Pi$, if $y_k$ and $y_l$ are in the same cluster in $\pi^m$, then $x_i$ and $x_j$ are in the same cluster in $\pi^m$. Thus we have $b_{ij}\geq \tilde{b}_{kl}$. If $y_k$ and $y_l$ are in different clusters in $\pi^m$, then $x_i$ and $x_j$ are in different clusters. Thus we have $b_{ij}\leq\tilde{b}_{kl}$. Because $b_{ij}\geq \tilde{b}_{kl}$ and $b_{ij}\leq\tilde{b}_{kl}$, we have $b_{ij}=\tilde{b}_{kl}$, which leads to $a_{ij}=\tilde{a}_{kl}$.
\end{proof}

Similar to the conventional CA matrix \cite{Fred05_EAC} (see Definition~\ref{def:CA}), the microcluster based co-association (MCA) matrix is computed by considering how many times two microclusters occur in the same cluster in $\Pi$. Then the microcluster similarity graph (MSG) is constructed based on the MCA matrix (see Definition~\ref{def:MSG}). By using the microclusters, the size of the similarity graph is reduced from $N$ to $\tilde{N}$.

\begin{mydef}
\label{def:MSG}
The microcluster similarity graph (MSG) is defined as
\begin{equation}
\tilde{G}=(\tilde{V},\tilde{L}),
\end{equation}
where $\tilde{V}=\mathcal{Y}$ is the node set and $\tilde{L}$ is the link set. The weight of the link between the nodes $y_i$ and $y_j$ is defined as
\begin{equation}
\tilde{w}_{ij}=\tilde{a}_{ij}.
\end{equation}
\end{mydef}

\subsection{Elite Neighbor Selection}
\label{sec:ENS}
In this section, we introduce the elite neighbor selection (ENS) strategy to deal with the problem of uncertain links. The uncertain links are the connections in the similarity graph that are of low confidence (typically with small weights). One key issue here is how to decide a \emph{proper} threshold to classify low-confidence and high-confidence and thereby identify the uncertain links in the graph. Yi et al. \cite{yi_icdm12} proposed to use global thresholds to identify the uncertain entries in the similarity matrix, which, however, has several drawbacks in practical applications. First, it neglects the local structure of the similarity graph and may lead to isolated nodes. Second, it is also a difficult task to find a proper global threshold for different clustering ensembles due to the inherent complexity of real-world datasets. Instead of using global thresholds, in this paper, we propose a $k$-nearest-neighbor-like strategy, termed ENS, to identify the uncertain links by locally adaptive thresholds.

\begin{mydef}
\label{def:K_ENT}
The $K$-elite neighbor threshold for a node $y_i$ in the MSG is defined as the value of the $K$-th largest link weight connected to $y_i$, which is denoted as $Thres_{K}(y_i)$.
\end{mydef}

Having constructed the MSG graph (see Definition~\ref{def:MSG}), we define the $K$-elite neighbor threshold, denoted as $Thres_{K}(y_i)$, for each node $y_i$ in the graph (see Definition~\ref{def:K_ENT}). Given two node $y_i$ and $y_j$ in the MSG graph, if $y_i$ is one of the top-$K$ neighbors of $y_j$ OR $y_j$ is one of the top-$K$ neighbors of $y_i$, then $y_i$ and $y_j$ are referred to as the $K$-elite neighbors for each other. Formally, the definition of the $K$-elite neighbors is given in Definitions~\ref{def:K_EN}. Note that the $K$-elite neighbor relationship is symmetric, i.e., $y_i\in K$-EN($y_j$) is equivalent to $y_j\in K$-EN($y_i$) (see Theorem~\ref{thm:K-EN_sym}).

\begin{mydef}
\label{def:K_EN}
Given two nodes $y_i$ and $y_j$ in the MSG graph, $y_i$ is a $K$-elite neighbor ($K$-EN) of $y_j$ if and only if $\tilde{w}_{ij}\geq$ $Thres_{K}(y_i)$ or $\tilde{w}_{ij}\geq$ $Thres_{K}(y_j)$. The set of the $K$-elite neighbors ($K$-ENs) of $y_i$ is denoted as $K$-EN($y_i$).
\end{mydef}

\begin{mythm}
\label{thm:K-EN_sym}
The $K$-elite neighbor relationship is symmetric, i.e., for all $y_i,y_j\in\mathcal{Y}$, $y_i$ is a $K$-elite neighbor of $y_j$ if and only if $y_j$ is a $K$-elite neighbor of $y_i$ .
\end{mythm}
\begin{proof}
Given $y_i\in K$-EN($y_j$), it holds that $\tilde{w}_{ij}\geq$ $Thres_{K}(y_i)$ or $\tilde{w}_{ij}\geq$ $Thres_{K}(y_j)$. Thus we have $y_j\in K$-EN($y_i$) according to Definition~\ref{def:K_EN}. Given $y_i\not\in K$-EN($y_j$), it holds that $\tilde{w}_{ij}<$ $Thres_{K}(y_i)$ and $\tilde{w}_{ij}<$ $Thres_{K}(y_j)$. Thus we have $y_j\not\in K$-EN($y_i$).
\end{proof}

\begin{mydef}
\label{def:K_ENG}
The $K$-elite neighbor graph ($K$-ENG) is defined as
\begin{equation}
\bar{G}=(\bar{V},\bar{L}),
\end{equation}
where $\bar{V}=\mathcal{Y}$ is the node set and $\bar{L}$ is the link set. The weight of the link between the nodes $y_i$ and $y_j$ is defined as
\begin{equation}
\bar{w}_{ij}=\begin{cases}
\tilde{w}_{ij}, &\text{if $y_i\in K$-EN($y_j$),}\\
0,&\text{otherwise}.
\end{cases}
\end{equation}
\end{mydef}

The $K$-elite neighbor graph ($K$-ENG) is constructed by preserving a certain number of probably reliable links in the MSG. The link between two nodes, say, $y_i$ and $y_j$, is preserved if and only if $y_i$ is a $K$-elite neighbor of $y_j$, i.e., $y_i\in K$-EN($y_j$). In our experimental study, we have shown that setting $K$ to a small value, e.g., in the interval of $[5,20]$, which preserves only several percent or even less than one percent of the links in the MSG, can lead to much better and more robust clustering results than preserving a great portion or even all of the links. In the following steps, the small number of probably reliable links are exploited by the random walk process and a dense pair-wise measure is derived from the sparse graph $K$-ENG.

\subsection{From Sparse Graph to Dense Similarity}
\label{sec:link_propagation}
To discover the cluster structure from the sparse graph $K$-ENG, we use the random walks to explore the graph and propose to derive a dense pair-wise similarity measure based on the probability trajectories of random walkers. A random walker is a dynamic process that randomly transits from one node to one of its neighbors with a certain transition probability. The random walk technique has been widely used in the field of community discovery \cite{newman04,pons_rw_05,lai_PRE10,neiwalk14_tkde} due to its ability of finding community structures, or cluster structures, in a graph. As pointed out in \cite{lai_PRE10}, the random walkers that start from the same cluster are more likely to have similar patterns when visiting the graph. In other words, the random walkers that start from the same cluster are more likely to have similar trajectories when they randomly walk on the graph than the random walkers that start from different clusters. Based on the random walk technique, in this paper, we propose to discover the latent relationships between graph nodes by analyzing the probability trajectories of the random walkers that start from different nodes.

The nodes in the $K$-ENG are microclusters, each of which consists of a certain number of original objects. Here, we refer to the number of objects in a microcluster as the size of the microcluster. In a link-weighted graph, the transition probability from a node to one of its neighbors is generally proportional to the weight of the link between them \cite{newman04,pons_rw_05,lai_PRE10}. In the graph model of \cite{newman04}, \cite{pons_rw_05}, and \cite{lai_PRE10}, all of the nodes are identical, which is different from the $K$-ENG in that the nodes in $K$-ENG may be of different sizes or even significantly different sizes. There is a need to distinguish the nodes of different sizes when deciding the transition probabilities.

\begin{figure}[!t]
\begin{center}
{\subfigure[]
{\includegraphics[width=0.4\linewidth]{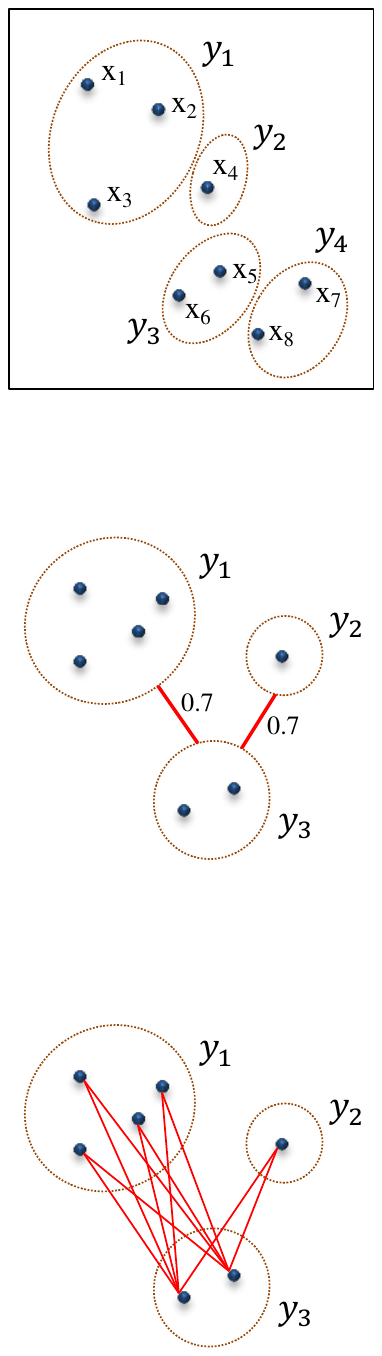}\label{fig:node_weight1}}}
{\subfigure[]
{\includegraphics[width=0.4\linewidth]{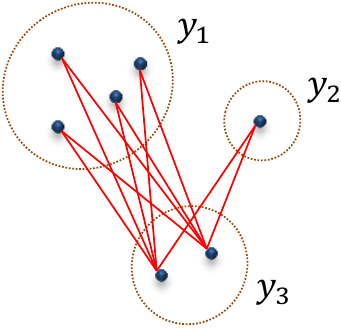}\label{fig:node_weight2}}}
\caption{Illustration of a graph with three microclusters as nodes. (a) The microcluster-microcluster links. (b) The (hidden) object-object links.}
\label{fig:node_weights}
\end{center}
\end{figure}

In Fig.~\ref{fig:node_weight1}, we illustrate a graph with three nodes, each being a microcluster. The nodes $y_1$, $y_2$, and $y_3$ consist of four, one, and two objects, respectively. There are two (microcluster-microcluster) links in this graph which are equally weighted at $0.7$. When treating the three nodes equally, the transition probabilities from $y_3$ to $y_1$ and from $y_3$ to $y_2$ are equally $0.5$, i.e., $p_{31}=p_{32}=0.7/(0.7+0.7)=0.5$. However, the microcluster-microcluster links are representative of a certain number of hidden object-object links. With respect to the object-object relationships, there are $4\times 2=8$ and $1\times 2 = 2$ hidden object-object links between $y_1$ and $y_3$ and between $y_2$ and $y_3$, respectively (see Fig.~\ref{fig:node_weight2}). The weights of the object-object links correspond to the entries of the CA matrix (see Definition~\ref{def:CA}), whereas the weights of the microcluster-microcluster links correspond to the entries of the MCA matrix (see Definition~\ref{def:MCA}). According to Theorem~\ref{thm:MCA_CA}, the weight of every object-object link between $y_1$ and $y_3$ (or between $y_2$ and $y_3$) is equal to that of the microcluster-microcluster link between $y_1$ and $y_3$ (or between $y_2$ and $y_3$), that is to say, the weight of every object-object link in the graph illustrated in Fig.~\ref{fig:node_weight2} is equal to $0.7$. If we perform random walks on the object granularity rather than the microcluster granularity, the probability of walking from one of the objects in $y_3$ to one of the objects in $y_1$ would be four times as great as the probability of walking from one of the objects in $y_3$ to one of the objects in $y_2$.

To reflect the hidden object-object connections in a microcluster based graph, the construction of the transition probability matrix needs to take into consideration the sizes of the microclusters. Specifically, the transition probability from a microcluster node, say, $y_i$, to one of its neighbors should be proportional to the size of this neighbor if the weights of the links between $y_i$ and all of its neighbors are equal. In our scenario, both the weights between $y_i$ and its neighbors and the sizes of its neighbors may be different. Let $\tilde{n}_i$ be the size of a microcluster $y_i$. The transition probability from $y_i$ to one of its neighbors, say, $y_j$, is defined to be proportional to the sum of the weights of all hidden object-object links between $y_i$ and $y_j$, i.e., $\tilde{n}_i\cdot \tilde{n}_j\cdot \bar{w}_{ij}$. Because $\tilde{n}_i$ is a constant given the node $y_i$, the probability of walking from $y_i$ to one of its neighbors $y_j$ therefore should be proportional to $\tilde{n}_j\cdot \bar{w}_{ij}$. Formally, the definition of the transition probability matrix on the $K$-ENG is given in Definition~\ref{def:transit_matrix}.

\begin{mydef}
\label{def:transit_matrix}
Let $P=\{p_{ij}\}_{\tilde{N}\times\tilde{N}}$ be the transition probability matrix of the random walk on the $K$-ENG. The transition probability from node $y_i$ to node $y_j$ is defined as
\begin{equation}
\label{eq:transit_matrix}
p_{ij}=\frac{\tilde{n}_j\cdot\bar{w}_{ij}}{\sum_{k\neq i}\tilde{n}_k\cdot\bar{w}_{ik}},
\end{equation}
where $\tilde{n}_j$ is the number of objects in $y_j$.
\end{mydef}

The random walk process is driven by the transition probability matrix P. Let $P^T=\{p^T_{ij}\}_{\tilde{N}\times\tilde{N}}$ be the $T$-step transition probability matrix, i.e., probability distribution at step $T$, where $p^T_{ij}$ is the probability that a random walker starting from node $y_i$ arrives at node $y_j$ at step $T$. The probability distribution at step $1$ is obviously the transition probability matrix, i.e., $P^1=P$. The probability distribution at step $T$ is computed as $P^T=P\cdot P^{T-1}$, for $T\geq 2$.

Let $p^T_{i:}=\{p^T_{i1},\cdots,p^T_{i\tilde{N}}\}$ denote the probability distribution of node $y_i$ at step $T$, which is the $i$-th row of $P^T$ and represents the probability of going from node $y_i$ to each node in the graph by the random walk process at step $T$. The relationship between node $y_i$ and node $y_j$ can be studied by comparing their probability distributions at a certain step \cite{pons_rw_05}. However, the probability distributions at different steps reflect different scales of information for the graph structure. Using a single step of probability distribution as the feature of a node overlooks the properties of this node at different scales. In order to take advantage of multi-scale information in the graph, we propose to exploit the probability trajectory for describing the random walk process starting from each node, which considers the probability distributions from step $1$ to step $T$ rather than a single step. The formal definition of the probability trajectory is given in Definition~\ref{def:PT}.

\begin{mydef}
\label{def:PT}
The probability trajectory of a random walker starting from node $y_i$ with length $T$ is defined as a $T\tilde{N}$-tuple:
\begin{equation}
PT^T_i=\{p^1_{i:},p^2_{i:},\cdots,p^T_{i:}\},
\end{equation}
where $p^T_{i:}$ is the probability distribution of node $y_i$ at step $T$.
\end{mydef}

The probability trajectory of a random walker starting from a given node is a $T\tilde{N}$-tuple and can be viewed as a feature vector for the node. We further define a pair-wise similarity measure based on the probability trajectory representation.

\begin{mydef}
\label{def:PTS}
The probability trajectory based similarity (PTS) between node $y_i$ and node $y_j$ is defined as
\begin{equation}
\label{eq:PTS}
PTS_{ij}=Sim(PT^T_i,PT^T_j),
\end{equation}
where $Sim(u,v)$ is a similarity measure between two vectors $u$ and $v$.
\end{mydef}

In fact, any similarity measure can be used in Eq.~(\ref{eq:PTS}). In our work, we use the cosine similarity as the similarity measure, which effectively captures the relationship between the random walk trajectories. The cosine similarity between two vectors $u$ and $v$ is computed as follows:
\begin{equation}
\label{eq:cosine}
Sim_{cos}(u,v)=\frac{<u,v>}{\sqrt{<u,u>\cdot <v,v>}},
\end{equation}
where $<u,v>$ is the inner product of $u$ and $v$.

Therefore we obtain the new similarity measure PTS using the probability trajectory of the random walker starting from each node. Specifically, the random walks are performed on the sparse graph $K$-ENG which preserves only a small number of probably reliable links by the ENS strategy. The probability trajectories on the $K$-ENG are used as the feature vectors for the graph nodes, which incorporates multi-scale graph information into a $T\tilde{N}$-tuple by the different steps of the random walks. Theoretically, it is possible that the $K$-ENG may consist of more than one connected component, in which case we can perform the random walk on each connected component of the graph separately and then map the random walk trajectories at each component back to the whole graph to facilitate the computation.

For clarity, the algorithm of computing PTS is given in Algorithm 1.

\begin{figure}[!htb]
\textbf{Algorithm 1 (Computation of Probability Trajectory Based Similarity)}\\
\small{ {\bfseries Input:} $\Pi$, $k$.
\begin{algorithmic}[1]
    \STATE Obtain the set of microclusters $\mathcal{Y}$ from $\Pi$.\\
    \STATE Compute the MCA matrix by Eq.~(\ref{eq:MCA}).\\
    \STATE Build the graph MSG using the MCA matrix.\\
    \STATE Construct the sparse graph $K$-ENG by the ENS strategy.\\
    \STATE Perform random walks on $K$-ENG with the transition probability matrix given in Eq.~(\ref{eq:transit_matrix}).\\
    \STATE Compute the new similarity PTS by Eq.~(\ref{eq:PTS}).\\
\end{algorithmic}
{\bfseries Output:} $\mathcal{Y}$, $\{PTS_{ij}\}_{\tilde{N}\times\tilde{N}}$.}
\end{figure}

\subsection{Consensus Functions}

Having generated the new similarity measure PTS, the next step is to obtain the consensus clustering. Here, any clustering algorithm based on pair-wise similarity can be applied to the PTS measure to obtain the final clustering. Typically, we propose two different types of consensus functions based on PTS, termed probability trajectory accumulation (PTA) and probability trajectory based graph partitioning (PTGP), respectively.

\subsubsection{Probability Trajectory Accumulation (PTA)}
\label{sec:PTA}
In this section, we introduce the consensus function termed PTA, which is based on hierarchical agglomerative clustering.

To perform hierarchical clustering, the set of microclusters are treated as the initial regions and the PTS is used as the similarity measure to guide the region merging process. Let $\mathcal{R}^{(0)}=\{R^{(0)}_1,\cdots,R^{(0)}_{|\mathcal{R}^{(0)}|}\}$ denote the set of initial regions, where $R^{(0)}_j = y_j$ and $|\mathcal{R}^{(0)}|=\tilde{N}$. Let $S^{(0)}=\{s^{(0)}_{ij}\}_{|\mathcal{R}^{(0)}|\times |\mathcal{R}^{(0)}|}$ be the initial similarity matrix, where $s^{(0)}_{ij}=PTS_{ij}$.

In each step, the two regions with the highest similarity are merged into a new and bigger region and thus the number of regions decrements by one. Then the similarity matrix for the new set of regions will be computed w.r.t. average-link (AL), complete-link (CL), or single-link (SL). Let $\mathcal{R}^{(t)}=\{R^{(t)}_1,\cdots,R^{(t)}_{|\mathcal{R}^{(t)}|}\}$ be the set of generated regions in the $t$-step, for $t=1,2,\cdots,\tilde{N}-1$, where $|\mathcal{R}^{(t)}|$ is the number of regions in $\mathcal{R}^{(t)}$. Note that each region contains one or more microclusters. We write it as $y_i\in R^{(t)}_j$ if microcluster $y_i$ is in region $R^{(t)}_j$. Let $S^{(t)}=\{s^{(t)}_{ij}\}_{|\mathcal{R}^{(t)}|\times |\mathcal{R}^{(t)}|}$ be the similarity matrix for $\mathcal{R}^{(t)}$, which can be computed w.r.t. AL, CL, or SL. That is

\begin{equation}
s^{(t)}_{ij}=\begin{cases}
\frac{1}{|R^{(t)}_i|\cdot |R^{(t)}_j|}\sum_{y_k\in R^{(t)}_i,y_l\in R^{(t)}_j}PTS_{kl}, &\text{If Method=AL,}\\
\sum_{y_k\in R^{(t)}_i,y_l\in R^{(t)}_j}PTS_{kl}, &\text{If Method=CL,}\\
\max_{y_k\in R^{(t)}_i,y_l\in R^{(t)}_j}PTS_{kl}, &\text{If Method=SL,}\\
\end{cases}
\end{equation}
where $|R^{(t)}_i|$ is the number of microclusters in $R^{(t)}_i$.

The region merging process is performed iteratively and the number of regions decrements by one in each step. Obviously, after the $(\tilde{N}-1)$-step, there will be one region left, which contains the entire set of the microclusters. Then we have a dendrogram, i.e., a hierarchical representation of clusterings. Each level in the dendrogram represents a clustering with a certain number of clusters (or regions). The final clustering is obtained by specifying a level for the dendrogram.

An advantage of agglomerative clustering is that it can efficiently generate a hierarchy of clusterings where each level represents a clustering with a certain number of clusters. However, the region merging process is inherently local and greedy. A mistaken merging may lead to increasing errors in the following merging steps. The similarity measure determines the region merging order and plays a crucial role in agglomerative clustering. In our work, the PTS measure is able to deal with the uncertain links and incorporate the global structure information in the ensemble, which is beneficial for improving the accuracy and robustness of the agglomerative clustering. The experimental results also show the advantage of the PTA method (base on PTS) compared to other pair-wise similarity based methods \cite{Fred05_EAC,iam_on11_linkbased,yi_icdm12,iamon08_icds,huang14_weac} (see Section~\ref{sec:comp_ensemb} and Table~\ref{table:compare_ce}).

For clarity, the PTA algorithm is summarized in Algorithm 2.

\begin{figure}[!htb]
\textbf{Algorithm 2 (Probability Trajectory Accumulation)}\\
\small{ {\bfseries Input:} $\Pi$, $k$.
\begin{algorithmic}[1]
    \STATE Compute the microclusters and the PTS measure according to Algorithm 1.\\
    \STATE Initialize the set of regions $\mathcal{R}^{(0)}$.
    \STATE Construct the dendrogram iteratively:\\
    \textbf{for} {$t=1,2,\cdots,\tilde{N}-1$}\\
     ~~~~Merge the most similar two regions in $\mathcal{R}^{(t-1)}$ w.r.t. $S^{(t-1)}$.\\
     ~~~~Obtain the new set of regions $\mathcal{R}^{(t)}$.\\
     ~~~~Compute the new similarity matrix $S^{(t)}$.\\
    \textbf{end for}\\
    \STATE Find the clustering with $k$ clusters in the dendrogram.\\
    \STATE Obtain the final clustering by mapping microclusters back to objects.\\
\end{algorithmic}
{\bfseries Output:} the consensus clustering $\pi^*$.}
\end{figure}

\subsubsection{Probability Trajectory Based Graph Partitioning (PTGP)}
\label{sec:PTGP}
In this section, we introduce the consensus function termed PTGP, which is based on bipartite graph formulation.

A bipartite graph is constructed by treating both clusters and microclusters as nodes. As illustrated in Fig.~\ref{fig:MCBG}, there are no links between two clusters or between two microclusters. A link between two nodes exists if and only if one of the nodes is a cluster and the other is a microcluster. The weight of the link between a microcluster and a cluster is decided by the similarity between them. Here, we define the similarity between a microcluster $y_i$ and a cluster $C_j$ as the average PTS measure between $y_i$ and the microclusters in $C_j$. The formal definition is given as follows.

\begin{figure}[!t]
\begin{center}
{
{\includegraphics[width=0.8\linewidth]{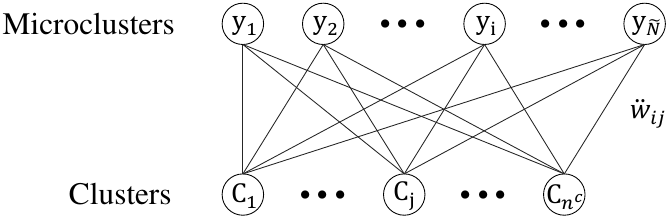}}}
\caption{The microcluster-cluster bipartite graph (MCBG)}
\label{fig:MCBG}
\end{center}
\end{figure}

\begin{mydef}
\label{def:mc_similarity}
The similarity between a microcluster $y_i$ and a cluster $C_j$ is defined as
\begin{align}
\label{eq:mc_similarity}
Sim_{mc}(y_i, C_j) &= Sim_{mc}(C_j, y_i)\nonumber\\
&=\frac{1}{|C_j|}\sum_{y_k\in C_j}PTS_{ik},
\end{align}
where $|C_j|$ is the number of microclusters in $C_j$.
\end{mydef}

\begin{mydef}
\label{def:MCBG}
The microcluster-cluster bipartite graph (MCBG) is defined as
\begin{equation}
\ddot{G}=(\ddot{V},\ddot{L}),
\end{equation}
where $\ddot{V}=\mathcal{Y}\bigcup \mathcal{C}=\{\ddot{v}_1,\dots,\ddot{v}_{\ddot{N}}\}$ is the node set, $\ddot{N}=\tilde{N}+N_c$ is the number of nodes and $\ddot{L}$ is the link set. The weight of the link between nodes $\ddot{v}_i$ and $\ddot{v}_j$ is defined as
\begin{equation}
\label{eq:mcbg_weight}
\ddot{w}_{ij}=\begin{cases}
Sim_{mc}(\ddot{v}_i,\ddot{v}_j), &\text{if } \ddot{v}_i\in \mathcal{Y}, \ddot{v}_j\in\mathcal{C}\\
&\text{or }\ddot{v}_i\in\mathcal{C}, \ddot{v}_j\in \mathcal{Y}, \\
0,&\text{otherwise}.
\end{cases}
\end{equation}
\end{mydef}

By treating both clusters and microclusters as nodes, the microcluster-cluster bipartite graph (MCBG) is constructed according to Definition~\ref{def:MCBG}. With regard to the bipartite structure of the MCBG, the efficient graph partitioning method Tcut \cite{CVPR12_Li} can be used to partition the graph into a certain number of disjoint sets of nodes. The microcluster nodes in the same segment are treated as a cluster and thus the final consensus clustering can be obtained by mapping the microclusters back to the data objects.

For clarity, we summarize the PTGP algorithm in Algorithm 3.

\begin{figure}[!htb]
\textbf{Algorithm 3 (Probability Trajectory Based Graph Partitioning)}\\
\small{ {\bfseries Input:} $\Pi$, $k$.
\begin{algorithmic}[1]
    \STATE Compute the microclusters and the PTS measure according to Algorithm 1.\\
    \STATE Compute the microcluster-cluster similarity by Eq.~(\ref{eq:mc_similarity}).\\
    \STATE Build the bipartite graph MCBG.
    \STATE Partition the MCBG into $k$ clusters using Tcut.
    \STATE Obtain the final clustering by mapping microclusters back to objects.
\end{algorithmic}
{\bfseries Output:} the consensus clustering $\pi^*$.}
\end{figure}

\section{Experiments}
\label{sec:experiment}

In this section, we conduct experiments using ten real-world datasets. All of the experiments are conducted in Matlab R2014a 64-bit on a workstation (Windows Server 2008 R2 64-bit, 8 Intel 2.40 GHz processors, 96 GB of RAM).

\subsection{Datasets and Evaluation Method}
\label{sec:dataset_and_eval}

In our experiments, we use ten real-world datasets, namely, \emph{Multiple Features} (\emph{MF}), \emph{Image Segmentation} (\emph{IS}), \emph{MNIST}, \emph{Optical Digit Recognition} (\emph{ODR}), \emph{Landsat Satellite} (\emph{LS}), \emph{Pen Digits} (\emph{PD}), \emph{USPS}, \emph{Forest Covertype} (\emph{FC}), \emph{KDD99-10P} and \emph{KDD99}. The \emph{MNIST} dataset and the \emph{USPS} dataset are from \cite{lecun98} and \cite{Dueck_AP_PHDThesis:09}, respectively. The \emph{KDD99} dataset is from the UCI KDD Archive \cite{uci_kdd_archive99}, whereas \emph{KDD99-10P} is a $10\%$ subset of \emph{KDD99}. The other six datasets are from the UCI machine learning repository \cite{Bache+Lichman:2013}. The details of the benchmark datasets are given in Table~\ref{table:datasets}.

We use the normalized mutual information (NMI) \cite{strehl02} to evaluate the quality of the consensus clusterings,  which provides a sound indication of the shared information between two clusterings. Note that a higher NMI indicates a better test clustering.

\begin{table}[!t]%
\centering
\caption{Description of the benchmark datasets}\vskip -0.1 in
\label{table:datasets}
\begin{center}
\begin{tabular}{p{1.6cm}<{\centering}|p{1.5cm}<{\centering}p{1.3cm}<{\centering}p{1.3cm}<{\centering}}
\toprule
Dataset         &\#Object     &\#Attribute      &\#Class\\
\midrule
\emph{MF}                &2,000  &649      &10\\
\emph{IS}                &2,310  &19     &7\\
\emph{MNIST}                &5,000  &784      &10\\
\emph{ODR}                &5,620  &64      &10\\
\emph{LS}                &6,435  &36      &6\\
\emph{PD}           &10,992 &16     &10\\
\emph{USPS}                 &11,000    &256     &10\\
\emph{FC}                &11,340  &54      &7\\
\emph{KDD99-10P}    &49,402     &41     &23\\
\emph{KDD99}       &494,020     &41     &23\\
\bottomrule
\end{tabular}
\end{center}
\end{table}

\subsection{Construction of Ensembles}
\label{sec:construct_base}

To evaluate the effectiveness of the proposed approach over various combinations of base clusterings, we construct a pool of a large number of base clusterings. In our experiments, we use the $k$-means algorithm and the rival penalized competitive learning (RPCL) \cite{xu93_rpcl} algorithm to construct the base clustering pool. The $k$-means and RPCL algorithms are performed repeatedly with random initializations and parameters. The numbers of initial clusters for $k$-means and RPCL are randomly chosen in the interval of $[2,ub]$, where $ub=\min\{\sqrt{N}/2,50\}$ is the upper bound of the number of clusters and $N$ is the number of objects in the dataset. By running $k$-means and RPCL $100$ times respectively, a pool of $200$ base clusterings is obtained for each benchmark dataset.

For each run of the proposed methods and the baseline ensemble clustering methods, we generate the ensemble by randomly drawing $M$ base clusterings from the base clustering pool. Unless specially mentioned, the ensemble size $M=10$ is used in this paper. To rule out the factor of \emph{getting lucky occasionally}, the average performances of the proposed methods and the baseline methods are evaluated and compared over a large number of runs.

\subsection{Parameter Analysis}
\label{sec:sens_paras}

\begin{table*}[!thb]\footnotesize
\centering 
\caption{Average information of the MSG graph over 20 runs.}\vskip -0.05 in
\label{table:avg_MC} 
\begin{tabular}{m{2.3cm}<{\centering}|m{1cm}<{\centering}m{1cm}<{\centering}m{1.1cm}<{\centering}m{1.1cm}<{\centering}m{1.1cm}<{\centering}m{1.1cm}<{\centering}m{1.2cm}<{\centering}m{1.1cm}<{\centering}m{1.5cm}<{\centering}m{0.9cm}<{\centering}}
\toprule
$Dataset$  &\emph{MF}  &\emph{IS}  &\emph{MNIST}  &\emph{ODR}  &\emph{LS} &\emph{PD} &\emph{USPS} &\emph{FC}    &\emph{KDD99-10P}   &\emph{KDD99}\\
\midrule
\#Node (i.e.,$\tilde{N}$)   &$242$	&$297$	&$1,438$	&$899$	 &$1,064$	 &$1,095$	 &$2,975$	 &$1,837$    &$230$  &$301$\\
\midrule
\#Link    &$12,978$	&$19,778$	&$468,453$	&$153,949$	 &$219,753$	&$208,146$	 &$1,825,938$	 &$583,369$    &$19,260$   &$34,509$\\
\bottomrule
\end{tabular}
\end{table*}

\begin{table*}[!thb]\footnotesize
\centering 
\caption{Average RatioPL over 20 runs with varying parameter $K$.}\vskip -0.05 in
\label{table:RatioPL} 
\begin{tabular}{m{1.2cm}<{\centering}|m{1.5cm}<{\centering}|m{0.9cm}<{\centering}m{0.9cm}<{\centering}m{0.9cm}<{\centering}m{0.9cm}<{\centering}m{0.9cm}<{\centering}m{0.9cm}<{\centering}m{0.9cm}<{\centering}m{0.9cm}<{\centering}}
\toprule
\multicolumn{2}{c|}{$K$}  &1  &2  &4  &8  &16 &32 &64 &ALL\\
\midrule
\multirow{8}{*}{RatioPL}
&\emph{MF}	&$2.4\%$	&$3.4\%$	&$6.6\%$	&$11.4\%$	 &$21.1\%$	 &$43.5\%$	 &$82.1\%$	 &$100.0\%$\\
&\emph{IS}	&$2.0\%$	&$3.2\%$	&$5.9\%$	&$10.5\%$	 &$19.4\%$	 &$37.9\%$	 &$70.1\%$	 &$100.0\%$\\
&\emph{MNIST}	&$0.6\%$	&$0.9\%$	&$1.7\%$	&$2.9\%$	 &$5.0\%$	 &$8.7\%$	 &$16.2\%$	 &$100.0\%$\\
&\emph{ODR}	&$1.0\%$	&$1.6\%$	&$3.0\%$	&$4.9\%$	 &$8.5\%$	 &$15.0\%$	 &$31.0\%$	 &$100.0\%$\\
&\emph{LS}	&$0.8\%$	&$1.1\%$	&$2.3\%$	&$3.8\%$	 &$6.9\%$	 &$12.4\%$	 &$23.9\%$	 &$100.0\%$\\
&\emph{PD}	&$0.9\%$	&$1.2\%$	&$2.4\%$	&$4.0\%$	 &$7.1\%$	 &$13.5\%$	 &$28.1\%$	 &$100.0\%$\\
&\emph{USPS}	&$0.3\%$	&$0.5\%$	&$1.0\%$	&$1.7\%$	 &$2.9\%$	 &$5.0\%$	 &$8.7\%$	 &$100.0\%$\\
&\emph{FC}	&$0.5\%$	&$0.7\%$	&$1.5\%$	&$2.4\%$	 &$4.4\%$	&$8.0\%$	 &$15.2\%$	 &$100.0\%$\\
&\emph{KDD99-10P}   &$1.7\%$	&$2.1\%$	&$4.2\%$	&$9.2\%$	&$23.9\%$	 &$53.4\%$	&$90.0\%$	 &$100.0\%$\\
&\emph{KDD99}   &$1.6\%$	&$1.9\%$	&$3.6\%$	&$7.2\%$	&$14.3\%$	 &$34.0\%$	&$68.9\%$	 &$100.0\%$\\
\bottomrule
\end{tabular}
\end{table*}

\begin{table*}[!thb]\footnotesize
\centering 
\caption{Average performance (in terms of NMI) of PTA over 20 runs with varying parameters $K$ and $T$.}\vskip -0.05 in
\label{table:comp_para_PTA_KT} 
\begin{tabular}{m{1.5cm}<{\centering}|m{0.50cm}<{\centering}m{0.50cm}<{\centering}m{0.50cm}<{\centering}m{0.50cm}<{\centering}m{0.50cm}<{\centering}m{0.50cm}<{\centering}m{0.50cm}<{\centering}m{0.50cm}<{\centering}|m{0.50cm}<{\centering}m{0.50cm}<{\centering}m{0.50cm}<{\centering}m{0.50cm}<{\centering}m{0.50cm}<{\centering}m{0.50cm}<{\centering}m{0.50cm}<{\centering}m{0.50cm}<{\centering}}
\toprule
$K$  &1  &2  &4  &8  &16 &32 &64 &ALL   &\multicolumn{8}{c}{10}\\
\midrule
$T$ &\multicolumn{8}{c|}{10}    &1  &2  &4  &8  &16 &32 &64 &128\\
\midrule
\emph{MF}	&0.585	&0.597	&0.630	&0.627	&0.605	&0.560	 &0.538	&0.535	 &0.613	 &0.617	&0.619	 &0.620	 &0.620	 &0.620	 &0.619	 &0.615\\
\emph{IS}	&0.574	&0.612	&0.632	&0.615	&0.595	&0.592	 &0.610	&0.609	 &0.613	 &0.616	&0.611	 &0.612	 &0.611	 &0.612	 &0.611	 &0.612\\
\emph{MNIST}	&0.556	&0.575	&0.585	&0.584	&0.589	&0.574	 &0.538	&0.486	 &0.567	&0.585	&0.582	 &0.591	 &0.593	 &0.591	 &0.592	 &0.592\\
\emph{ODR}	&0.770	&0.792	&0.819	&0.820	&0.813	&0.798	 &0.757	&0.713	 &0.790	 &0.816	&0.813	 &0.817	 &0.816	 &0.817	 &0.812	 &0.810\\
\emph{LS}	&0.599	&0.595	&0.615	&0.618	&0.621	&0.612	 &0.586	&0.539	 &0.600	 &0.613	&0.618	 &0.620	 &0.620	 &0.627	 &0.632	 &0.637\\
\emph{PD}	&0.714	&0.715	&0.757	&0.765	&0.761	&0.735	 &0.698	&0.678	 &0.722	 &0.759	&0.768	 &0.765	 &0.762	 &0.762	 &0.761	 &0.760\\
\emph{USPS}	&0.560	&0.565	&0.585	&0.590	&0.591	&0.579	 &0.570	&0.442	 &0.574	 &0.587	&0.588	 &0.592	 &0.595	 &0.594	 &0.594	 &0.587\\
\emph{FC}	&0.235	&0.237	&0.247	&0.254	&0.250	&0.237	 &0.221	&0.199	 &0.246	 &0.245	&0.257	 &0.255	 &0.254	 &0.250	 &0.250	 &0.250\\
\emph{KDD99-10P}   &0.565	&0.599	&0.624	&0.659	&0.664	&0.592	&0.562	 &0.561	&0.677	&0.681	 &0.686	&0.677	 &0.671	 &0.668	 &0.676	 &0.670\\
\emph{KDD99}   &0.538	&0.560	&0.613	&0.682	&0.700	&0.644	&0.570	&0.562	 &0.687	&0.693	&0.698	 &0.700	&0.691	 &0.689	 &0.694	& 0.693\\
\bottomrule
\end{tabular}
\end{table*}

\begin{table*}[!thb]\footnotesize
\centering 
\caption{Average performance (in terms of NMI) of PTGP over 20 runs with varying parameters $K$ and $T$.}\vskip -0.05 in
\label{table:comp_para_PTGP_KT} 
\begin{tabular}{m{1.5cm}<{\centering}|m{0.50cm}<{\centering}m{0.50cm}<{\centering}m{0.50cm}<{\centering}m{0.50cm}<{\centering}m{0.50cm}<{\centering}m{0.50cm}<{\centering}m{0.50cm}<{\centering}m{0.50cm}<{\centering}|m{0.50cm}<{\centering}m{0.50cm}<{\centering}m{0.50cm}<{\centering}m{0.50cm}<{\centering}m{0.50cm}<{\centering}m{0.50cm}<{\centering}m{0.50cm}<{\centering}m{0.50cm}<{\centering}}
\toprule
$K$  &1  &2  &4  &8  &16 &32 &64 &ALL   &\multicolumn{8}{c}{10}\\
\midrule
$T$ &\multicolumn{8}{c|}{10}    &1  &2  &4  &8  &16 &32 &64 &128\\
\midrule
\emph{MF}	&0.609	&0.610	&0.627	&0.621	&0.596	&0.536	 &0.499	&0.495	 &0.619	 &0.617	&0.619	 &0.616	 &0.610	 &0.607	 &0.599	 &0.597\\
\emph{IS}	&0.610	&0.615	&0.623	&0.616	&0.589	&0.582	 &0.603	&0.594	 &0.615	 &0.618	&0.614	 &0.611	 &0.613	 &0.611	 &0.608	 &0.616\\
\emph{MNIST}	&0.575	&0.577	&0.585	&0.587	&0.587	&0.576	 &0.533	&0.471	 &0.581	&0.583	&0.585	 &0.586	 &0.588	 &0.589	 &0.589	 &0.587\\
\emph{ODR}	&0.809	&0.815	&0.822	&0.830	&0.821	&0.799	 &0.745	&0.679	 &0.812	 &0.817	&0.824	 &0.824	 &0.824	 &0.821	 &0.818	 &0.814\\
\emph{LS}	&0.604	&0.607	&0.615	&0.623	&0.623	&0.621	 &0.586	&0.505	 &0.615	 &0.616	&0.616	 &0.615	 &0.626	 &0.629	 &0.631	 &0.631\\
\emph{PD}	&0.746	&0.747	&0.759	&0.762	&0.756	&0.727	 &0.696	&0.649	 &0.749	 &0.751	&0.759	 &0.765	 &0.766	 &0.761	 &0.760	 &0.758\\
\emph{USPS}	&0.575	&0.570	&0.578	&0.586	&0.582	&0.572	 &0.556	&0.404	 &0.566	 &0.570	&0.576	 &0.580	 &0.580	 &0.588	 &0.584	 &0.574\\
\emph{FC}	&0.224	&0.227	&0.235	&0.238	&0.237	&0.234	 &0.218	&0.195	 &0.229	 &0.231	&0.232	 &0.236	 &0.237	 &0.237	 &0.236	 &0.234\\
\emph{KDD99-10P}   &0.618	&0.620	&0.627	&0.643	&0.646	&0.591	&0.524	 &0.521	&0.673	&0.675	 &0.665	&0.659	& 0.655	& 0.656	 &0.664	 &0.662\\
\emph{KDD99}   &0.620	&0.622	&0.624	&0.686	&0.701	&0.653	&0.571	&0.567	 &0.709	&0.713	&0.720	 &0.722	&0.721	 & 0.698	 & 0.696	 & 0.704\\
\bottomrule
\end{tabular}
\end{table*}

In this paper, we propose two ensemble clustering methods, termed PTA and PTGP, respectively. There are two parameters, namely, $K$ and $T$, in the proposed methods. The parameter $K$ specifies how many neighbors of a node will be treated as elite neighbors and preserved. The smaller the parameter $K$ is, the sparser the $K$-ENG will be. The parameter $T$ is the length of the probability trajectories.

We first test the influence of parameter $K$ on the proportion of preserved links in the $K$-ENG. The average information of the MSG over 20 runs on each dataset is shown in Table~\ref{table:avg_MC}. Because the microclusters are generated by intersecting multiple base clusterings, the number of microclusters is affected by the total number of objects as well as the shapes of cluster boundaries. For the \emph{KDD99-10P} and \emph{KDD99} datasets, the number of microclusters is less than $1\%$ of the number of the original objects. For the other datasets, the number of microclusters is averagely about $10\%$ to $20\%$ of the number of the original objects, i.e., using microclusters as nodes reduces the graph size by about $80\%$ to $90\%$.  The number of links in the MSG is also given in Table~\ref{table:avg_MC}. Note that a link between two nodes exists if and only if the weight between them is non-zero, i.e., a ``link'' with zero-weight does not count as a link here. By cutting out the probably unreliable links in the MSG via the ENS strategy, the sparse graph $K$-ENG is constructed. The ratio of preserved links (RatioPL) is defined as
\begin{equation}
\label{eq:ratioPL}
\mathrm{RatioPL} = \frac{\#\text{Links in $K$-ENG}}{\#\text{Links in MSG}}.
\end{equation}

For each parameter setting of $K$ and $T$, we run the proposed methods $20$ times with the ensemble of base clusterings randomly drawn from the base clustering pool (see Section~\ref{sec:construct_base}) at each time. The RatioPL with respect to different values of $K$ is shown in Table~\ref{table:RatioPL}. When $K=\text{ALL}$, all links in MSG are preserved. The average NMI scores of PTA and PTGP with varying $K$ and $T$ are reported in Tables~\ref{table:comp_para_PTA_KT} and \ref{table:comp_para_PTGP_KT}, respectively. The performances of PTA and PTGP are consistently good when $K$ is set in the interval of $[5,20]$  on the benchmark datasets and significantly better than setting $K=\text{ALL}$. As can be seen in Tables~\ref{table:RatioPL}, \ref{table:comp_para_PTA_KT} and \ref{table:comp_para_PTGP_KT}, preserving a small proportion of the links via the ENS strategy can lead to significantly better performance than using all graph links by setting $K=ALL$.

As shown in Tables~\ref{table:comp_para_PTA_KT} and \ref{table:comp_para_PTGP_KT}, the performances of the proposed PTA and PTGP methods are robust over various parameter settings. Setting $K$ and $T$ to moderate values, e.g., both in the interval of $[5,20]$,  leads to consistently good performances on the benchmark datasets. Empirically, it is suggested that the parameters $K$ and $T$ be set in the interval of $[\sqrt{\tilde{N}}/5,\sqrt{\tilde{N}}]$, where $\tilde{N}$ is the number of the graph nodes in the $K$-ENG. In the following of this paper, we set both $K$ and $T$ to the floor of $\sqrt{\tilde{N}}/2$ in all experiments on the benchmark datasets.

\subsection{Comparison against Base Clusterings}
\label{sec:comp_base}

The purpose of ensemble clustering is to combine multiple base clusterings into a probably better and more robust clustering. In this section, we compare the proposed PTA (associated with average-link) and PTGP methods against the base clusterings. Figure~\ref{fig:base_comp} illustrates the average NMI scores and the variances of the proposed methods and the base clusterings over 100 runs. As shown in Fig.~\ref{fig:base_comp}, for the benchmark datasets, the proposed PTA and PTGP algorithms produce overall more accurate clusterings than the base clusterings. Especially, for the \emph{ODR}, \emph{LS}, \emph{PD}, \emph{KDD99-10P}, and \emph{KDD99} datasets, the proposed methods achieve significant improvements in terms of NMI compared to the base clusterings.

\begin{figure}[!t]
\begin{center}
{
{\includegraphics[width=0.95\linewidth]{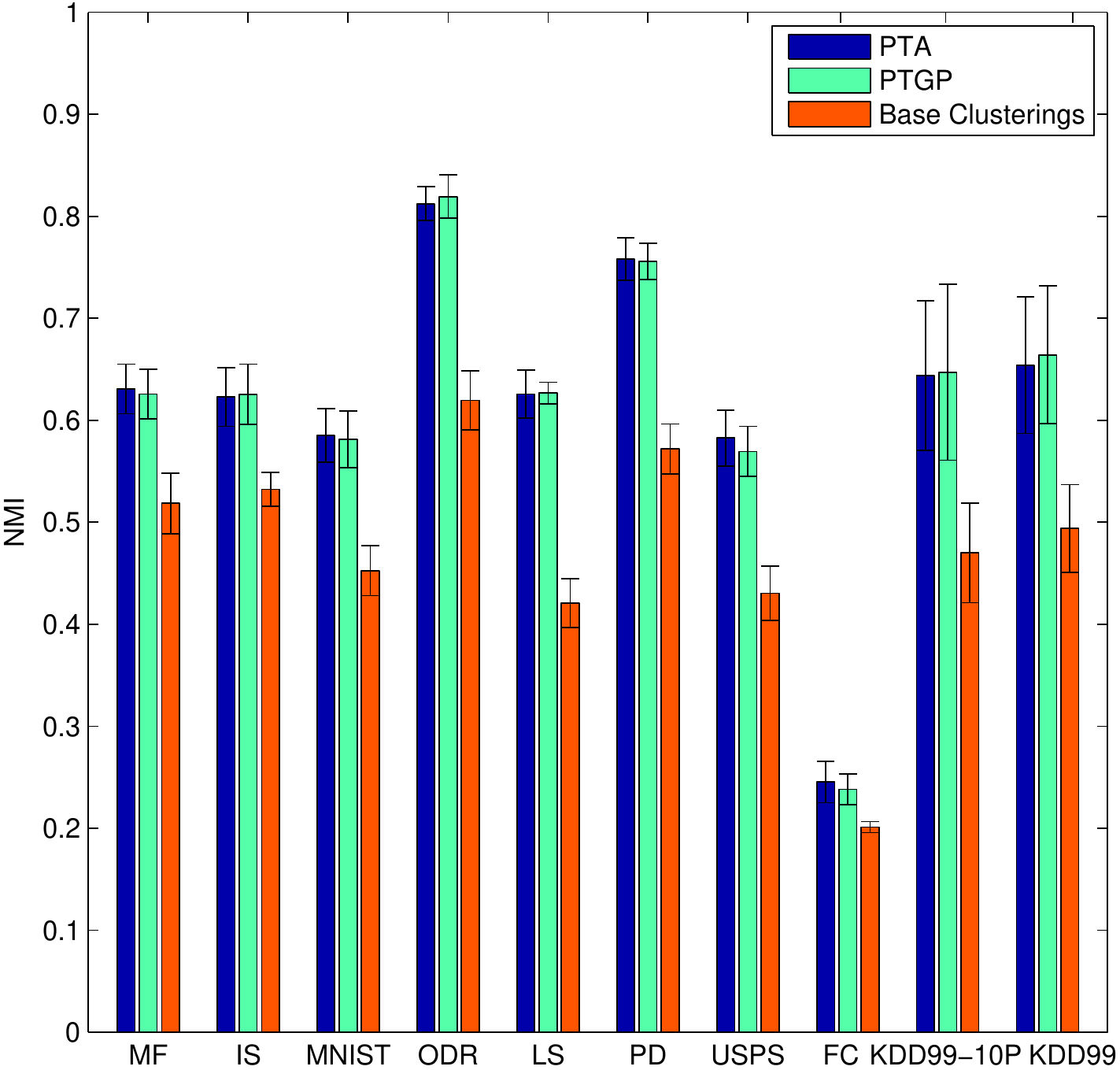}}} 
\caption{Average performances in terms of NMI of our methods and the base clusterings over $100$ runs.}
\label{fig:base_comp}
\end{center}
\end{figure}

\subsection{Comparison against Other Ensemble Clustering Approaches}
\label{sec:comp_ensemb}

In this section, we evaluate the effectiveness and robustness of the proposed PTA and PTGP methods by comparing them against ten ensemble clustering methods, namely, cluster-based similarity partitioning algorithm (CSPA) \cite{strehl02}, hypergraph partitioning algorithm (HGPA) \cite{strehl02}, meta-
clustering algorithm (MCLA) \cite{strehl02}, seeded random walk (SRW) \cite{abdala10_icpr}, graph partitioning with multi-granularity link analysis (GP-MGLA) \cite{huang14_weac}, evidence accumulation clustering (EAC) \cite{Fred05_EAC}, SimRank similarity based method (SRS) \cite{iamon08_icds}, weighted connected-triple method (WCT) \cite{iam_on11_linkbased}, ensemble clustering by matrix completion (ECMC) \cite{yi_icdm12}, and weighted evidence accumulation clustering (WEAC) \cite{huang14_weac}.

Of the ten baseline ensemble clustering methods, CSPA, HGPA, MCLA, SRW, and GP-MGLA are graph partitioning based methods, while EAC, SRS, WCT, ECMC, and WEAC are pair-wise similarity based methods. The pair-wise similarity based methods construct a similarity matrix based on the ensemble information. For each of the pair-wise similarity based methods and the proposed PTA method, we use three agglomerative clustering methods to obtain the final clusterings, namely, average-link (AL), complete-link (CL), and single-link (SL). Thus, each pair-wise similarity based method is associated with three sub-methods. For the other baseline methods, we use the parameter settings as suggested by the authors in the their papers \cite{Fred05_EAC,iam_on11_linkbased,strehl02,iamon08_icds,abdala10_icpr,huang14_weac}.

\begin{table*}[!t]\footnotesize
\centering 
\caption{Average performances (in terms of NMI) over 100 runs by different ensemble clustering methods (The three highest scores in each column are highlighted in bold)}\vskip -0.1 in
\label{table:compare_ce}
\begin{tabular}{|m{1.8cm}<{\centering}|m{0.96cm}<{\centering}m{0.96cm}<{\centering}|m{0.96cm}<{\centering}m{0.96cm}<{\centering}|m{0.96cm}<{\centering}m{0.96cm}<{\centering}|m{0.96cm}<{\centering}m{0.96cm}<{\centering}|m{0.96cm}<{\centering}m{0.96cm}<{\centering}|}
\hline
\multirow{2}{*}{Method}               &\multicolumn{2}{c|}{\emph{MF}}    &\multicolumn{2}{c|}{\emph{IS}}   &\multicolumn{2}{c|}{\emph{MNIST}} &\multicolumn{2}{c|}{\emph{ODR}} &\multicolumn{2}{c|}{\emph{LS}}\\
\cline{2-11}
&Best-$k$&True-$k$    &Best-$k$&True-$k$    &Best-$k$&True-$k$  &Best-$k$&True-$k$  &Best-$k$&True-$k$\\
\hline
\hline
PTA-AL  &\textbf{0.631}  &\textbf{0.614}  &0.623  &0.607  &\textbf{0.585}  &\textbf{0.578}  &\textbf{0.813}  &\textbf{0.804}  &\textbf{0.626}  &\textbf{0.622}\\
PTA-CL  &\textbf{0.629}  &\textbf{0.611}  &0.620  &\textbf{0.609}  &\textbf{0.582}  &\textbf{0.577}  &0.805  &0.793  &0.595  &0.584\\
PTA-SL  &0.612  &0.522  &0.616  &0.521  &0.523  &0.103  &0.760  &0.533  &0.552  &0.114\\
\hline
PTGP    &\textbf{0.626}  &\textbf{0.613}  &\textbf{0.625}  &\textbf{0.611}  &0.581  &\textbf{0.576}  &\textbf{0.819}  &\textbf{0.813}    &\textbf{0.627}  &\textbf{0.625}\\
\hline
\hline
CSPA    &0.597  &0.591  &0.605  &0.605  &0.493  &0.486  &0.726  &0.723    &0.511  &0.475\\
HGPA    &0.485  &0.231  &0.500  &0.457  &0.423  &0.120  &0.643  &0.353    &0.406  &0.324\\
MCLA    &0.617  &0.596  &0.623  &\textbf{0.609}  &0.536  &0.518  &0.785  &0.770    &0.550  &0.518\\
\hline
SRW     &0.468  &0.197  &0.509  &0.175  &0.393  &0.126  &0.514  &0.135     &0.434  &0.124\\
\hline
GP-MGLA &0.618  &0.604  &0.613  &0.608  &0.569  &0.557  &\textbf{0.807}  &\textbf{0.798} &0.615  &\textbf{0.607}\\
\hline
\hline
EAC-AL  &0.603  &0.578  &0.612  &0.605  &0.570  &0.555  &0.792  &0.772  &0.596  &0.569\\
EAC-CL  &0.572  &0.508  &0.622  &0.442  &0.460  &0.203  &0.651  &0.389  &0.459  &0.282\\
EAC-SL  &0.531  &0.173  &0.542  &0.413  &0.021  &0.002  &0.257  &0.099  &0.079  &0.002\\
\hline
SRS-AL  &0.613  &0.584  &0.614  &0.603  &0.575  &0.557  &0.794  &0.772  &0.603  &0.583\\
SRS-CL  &0.587  &0.547  &\textbf{0.625}  &0.585  &0.554  &0.534  &0.771  &0.744  &0.536  &0.453\\
SRS-SL  &0.484  &0.189  &0.579  &0.358  &0.017  &0.002  &0.130  &0.003  &0.065  &0.001\\
\hline
WCT-AL  &0.605  &0.579  &0.617  &0.606  &\textbf{0.585}  &0.562  &0.800  &0.774  &\textbf{0.617}  &0.603\\
WCT-CL  &0.579  &0.546  &\textbf{0.634}  &\textbf{0.612}  &0.561  &0.529  &0.774  &0.741  &0.540  &0.459\\
WCT-SL  &0.596  &0.247  &0.599  &0.415  &0.028  &0.002  &0.318  &0.132  &0.213  &0.002\\
\hline
ECMC-AL  &0.588  &0.333  &0.593  &0.277  &0.554  &0.132  &0.780  &0.328  &0.584  &0.018\\
ECMC-CL  &0.591  &0.520  &0.560  &0.411  &0.446  &0.276  &0.617  &0.454  &0.402  &0.204\\
ECMC-SL  &0.416  &0.166  &0.560  &0.200  &0.041  &0.015  &0.350  &0.177  &0.071  &0.001\\
\hline
WEAC-AL  &0.606  &0.583  &0.610  &0.605  &0.577  &0.569  &0.799  &0.785  &0.608  &0.596\\
WEAC-CL  &0.581  &0.520  &0.618  &0.431  &0.463  &0.194  &0.643  &0.371  &0.456  &0.235\\
WEAC-SL  &0.602  &0.268  &0.605  &0.417  &0.039  &0.002  &0.348  &0.126  &0.228  &0.002\\
\hline
\hline
\multirow{2}{*}{Method} &\multicolumn{2}{c|}{\emph{PD}}   &\multicolumn{2}{c|}{\emph{USPS}} &\multicolumn{2}{c|}{\emph{FC}} &\multicolumn{2}{c|}{\emph{KDD99-10P}} &\multicolumn{2}{c|}{\emph{KDD99}}\\
\cline{2-11}
&Best-$k$&True-$k$    &Best-$k$&True-$k$    &Best-$k$&True-$k$  &Best-$k$&True-$k$    &Best-$k$&True-$k$\\
\hline
\hline
PTA-AL  &\textbf{0.757}  &\textbf{0.732}  &\textbf{0.583}  &\textbf{0.565}  &\textbf{0.243}  &\textbf{0.232}    &\textbf{0.644}    &\textbf{0.525}    &\textbf{0.654}    &\textbf{0.510}\\
PTA-CL  &0.749  &\textbf{0.733}  &0.552  &0.530  &0.230  &\textbf{0.214}    &\textbf{0.659}    &0.509    &\textbf{0.683}    &0.489\\
PTA-SL  &0.700  &0.445  &0.499  &0.051  &\textbf{0.247}  &0.011    &0.636    &\textbf{0.545}    &0.635    &\textbf{0.545}\\
\hline
PTGP  &\textbf{0.755}  &\textbf{0.738}  &\textbf{0.568}  &\textbf{0.551}  &\textbf{0.239}  &\textbf{0.220}    &\textbf{0.647}    &\textbf{0.527}    &\textbf{0.664}    &\textbf{0.535}\\
\hline
\hline
CSPA  &0.669  &0.661  &0.481  &0.469  &0.213  &0.199  &N/A  &N/A  &N/A  &N/A\\
HGPA  &0.584  &0.193  &0.407  &0.017  &0.167  &0.103    &0.311    &0.155  &N/A  &N/A\\
MCLA  &0.699  &0.676  &0.519  &0.488  &0.229  &0.204    &0.622    &0.305    &0.621    &0.044\\
\hline
SRW  &0.469  &0.109  &0.467  &0.112  &0.198  &0.047  &N/A  &N/A  &N/A  &N/A\\
\hline
GP-MGLA  &\textbf{0.754}  &0.731  &0.560  &0.547  &0.227  &0.189    &0.631    &0.503    &0.623    &0.462\\
\hline
\hline
EAC-AL  &0.740  &0.699  &0.550  &0.526  &0.221  &0.194    &0.629    &0.510  &N/A  &N/A\\
EAC-CL  &0.615  &0.382  &0.424  &0.169  &0.207  &0.073    &0.601    &0.504  &N/A  &N/A\\
EAC-SL  &0.367  &0.013  &0.011  &0.001  &0.032  &0.002    &0.525    &0.218  &N/A  &N/A\\
\hline
SRS-AL  &N/A  &N/A  &N/A  &N/A  &N/A  &N/A  &N/A  &N/A  &N/A  &N/A\\
SRS-CL  &N/A  &N/A  &N/A  &N/A  &N/A  &N/A  &N/A  &N/A  &N/A  &N/A\\
SRS-SL  &N/A  &N/A  &N/A  &N/A  &N/A  &N/A  &N/A  &N/A  &N/A  &N/A\\
\hline
WCT-AL  &0.752  &0.695  &\textbf{0.563}  &0.533  &0.229  &0.199  &N/A  &N/A  &N/A  &N/A\\
WCT-CL  &0.694  &0.621  &0.516  &0.476  &0.217  &0.189  &N/A  &N/A  &N/A  &N/A\\
WCT-SL  &0.575  &0.015  &0.011  &0.001  &0.033  &0.001  &N/A  &N/A  &N/A  &N/A\\
\hline
ECMC-AL  &N/A  &N/A  &N/A  &N/A  &N/A  &N/A  &N/A  &N/A  &N/A  &N/A\\
ECMC-CL  &N/A  &N/A  &N/A  &N/A  &N/A  &N/A  &N/A  &N/A  &N/A  &N/A\\
ECMC-SL  &N/A  &N/A  &N/A  &N/A  &N/A  &N/A  &N/A  &N/A  &N/A  &N/A\\
\hline
WEAC-AL  &0.751  &0.716  &0.561  &\textbf{0.548}  &0.222  &0.194    &0.629    &0.499  &N/A  &N/A\\
WEAC-CL  &0.606  &0.349  &0.429  &0.146  &0.208  &0.065    &0.610    &0.502  &N/A  &N/A\\
WEAC-SL  &0.609  &0.029  &0.011  &0.001  &0.067  &0.002    &0.562    &0.275  &N/A  &N/A\\
\hline
\end{tabular}
\end{table*}

\begin{figure}[!t]
\begin{center}
{
{\includegraphics[width=0.98\columnwidth]{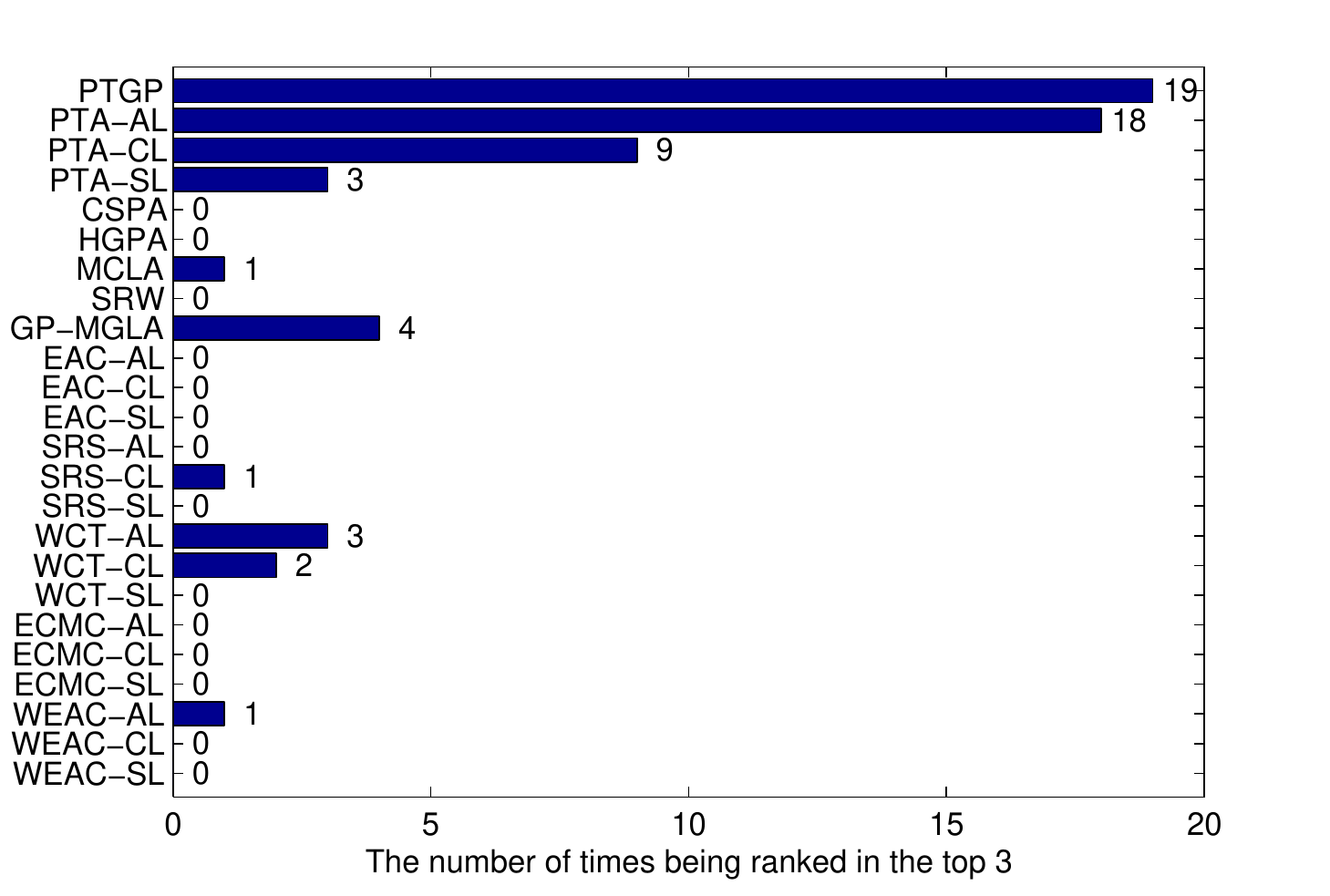}}}
\caption{The number of times of each method to be ranked in the top 3 (in terms of NMI) across the 20 columns in Table~\ref{table:compare_ce}.}
\label{fig:comp_top1top3}
\end{center}
\end{figure}

We run the proposed PTA and PTGP methods and the baseline methods 100 times on each dataset. If a method is computationally infeasible to be performed on a dataset, the corresponding NMI score will be labeled as "N/A". For each run, the ensemble of base clusterings is randomly drawn from the base clustering pool (see Section~\ref{sec:construct_base}). As the number of clusters of the consensus clustering needs to be pre-specified for the baseline methods and the proposed methods, to compare their clustering results in a fair way, we use two criteria to choose cluster numbers in the experiments, namely, best-$k$ and true-$k$. In the  best-$k$ criterion, the cluster number that leads to the best performance is specified for each method. In the true-$k$ criterion, the true number of classes of the dataset is specified for each method.

The average performances of the proposed PTA and PTGP methods and the baseline methods over 100 runs are reported in Table~\ref{table:compare_ce}. Each pair-wise similarity based method is associated with one of the three agglomerative clustering methods, namely, AL, CL, and SL. As shown in Table~\ref{table:compare_ce}, PTA-AL achieves the highest NMI scores for the \emph{MF}, \emph{MNIST}, and \emph{USPS} w.r.t. both best-$k$ and true-$k$ and almost the highest NMI scores for the \emph{ODR}, \emph{LS}, \emph{PD}, \emph{FC}, \emph{KDD99-10P}, and \emph{KDD99} datasets. The proposed PTGP method achieves the highest scores for the \emph{ODR} and \emph{LS} datasets w.r.t. both best-$k$ and true-$k$ and almost the highest scores for the \emph{MF}, \emph{IS}, \emph{PD}, \emph{USPS}, \emph{FC}, \emph{KDD99-10P}, and \emph{KDD99} datasets. To compare the performance of the test methods in a clearer way, Fig.~\ref{fig:comp_top1top3} shows the number of times of each method being ranked in the top 3 (in terms of NMI) in Table~\ref{table:compare_ce}. Out of the 20 columns in Table~\ref{table:compare_ce}, PTGP and PTA-AL are ranked in the top 3 (among the 24 test methods) 19 times and 18 times, respectively, while the best baseline method is ranked in the top 3 only 4 times (see Fig.~\ref{fig:comp_top1top3}). As can be seen in Table~\ref{table:compare_ce} and Fig.~\ref{fig:comp_top1top3}, the proposed PTA and PTGP methods achieve the overall best performance in clustering accuracy and robustness compared to the baseline methods across a variety of datasets.

It is worth mentioning that the PTGP method significantly outperforms the other five graph partitioning based methods, namely, CSPA, HGPA, MCLA, SRW, and GP-MGLA. Specifically, the PTGP method yields higher, or even significantly higher, NMI scores than the other graph partitioning based methods on all of the benchmark datasets (see Table~\ref{table:compare_ce}). Also, we compare the PTA method to the other five pair-wise similarity based methods, namely, EAC, SRS, WCT, ECMC, and WEAC, w.r.t. the same agglomerative clustering method. As shown in Table~\ref{table:compare_ce}, the PTA-AL method achieves the best performance among the pair-wise similarity based methods (all associated with AL) on all of the benchmark datasets. When considering CL or SL, the advantages of PTA become even greater. The PTA-CL method significantly outperforms the other pair-wise similarity based methods associated with CL on all benchmark datasets except \emph{IS}. The PTA-SL method achieves \emph{far} better consensus results than the other pair-wise similarity based methods associated with SL on all of the benchmark datasets. With the ability of handling uncertain links and incorporating global information to construct more accurate local links, the proposed PTA and PTGP methods perform significantly better than the baseline ensemble clustering methods on the benchmark datasets.

\begin{figure*}[!th]
\begin{center}
{\subfigure[\emph{MF}]
{\includegraphics[width=0.395\columnwidth]{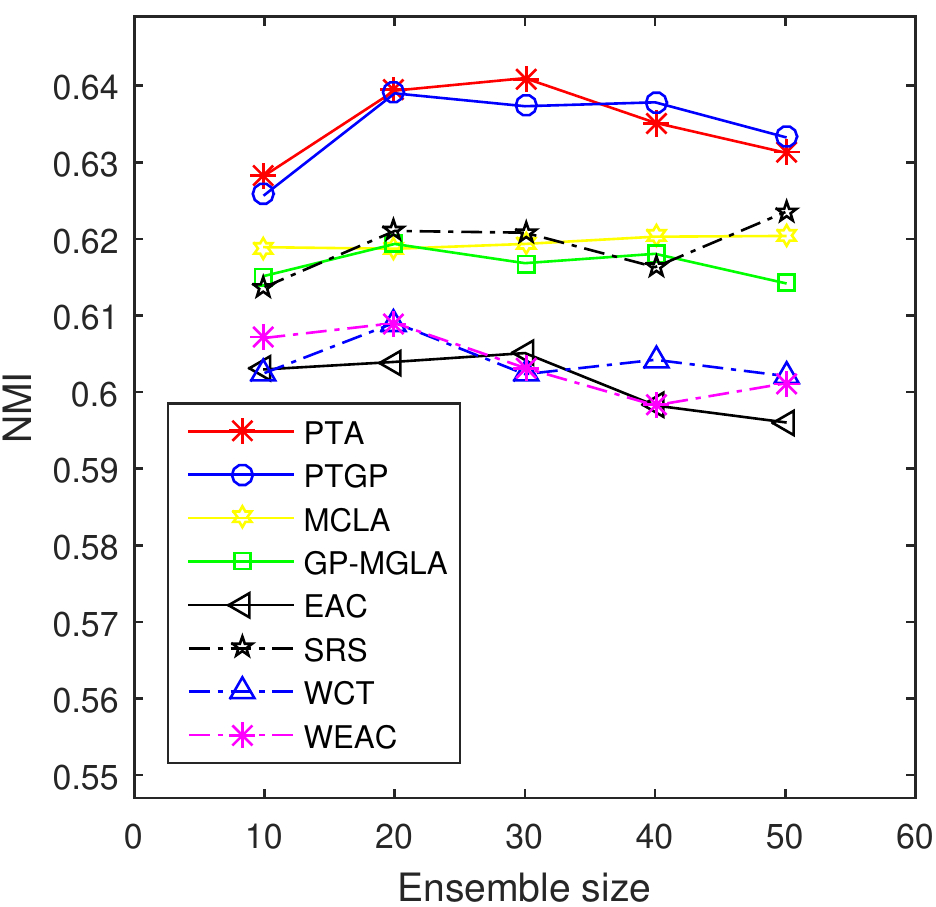}\label{fig:comp_Msize1}}}
{\subfigure[\emph{IS}]
{\includegraphics[width=0.395\columnwidth]{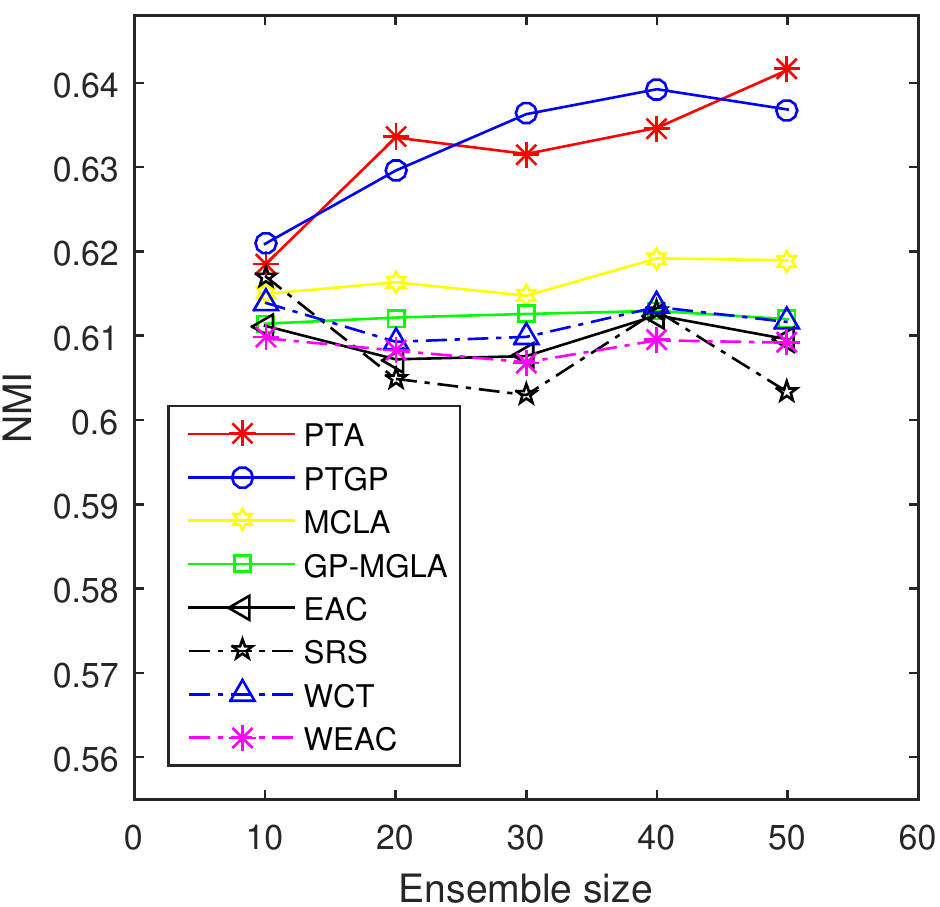}}}
{\subfigure[\emph{MNIST}]
{\includegraphics[width=0.395\columnwidth]{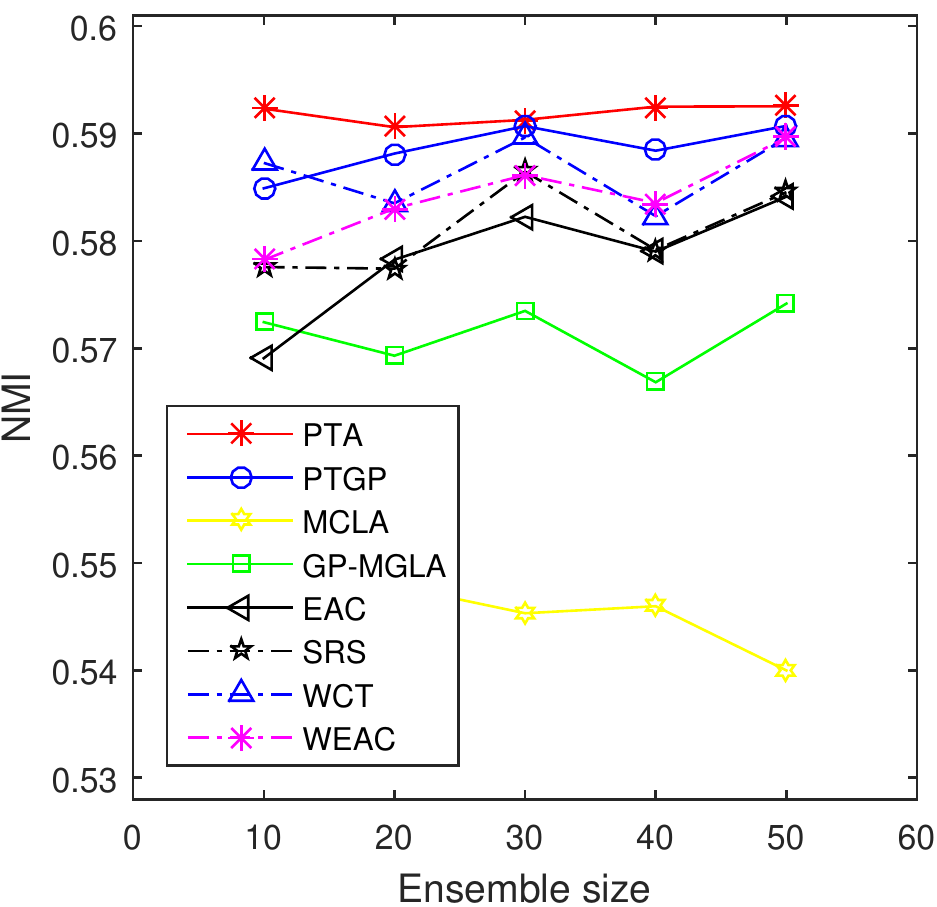}}}
{\subfigure[\emph{ODR}]
{\includegraphics[width=0.395\columnwidth]{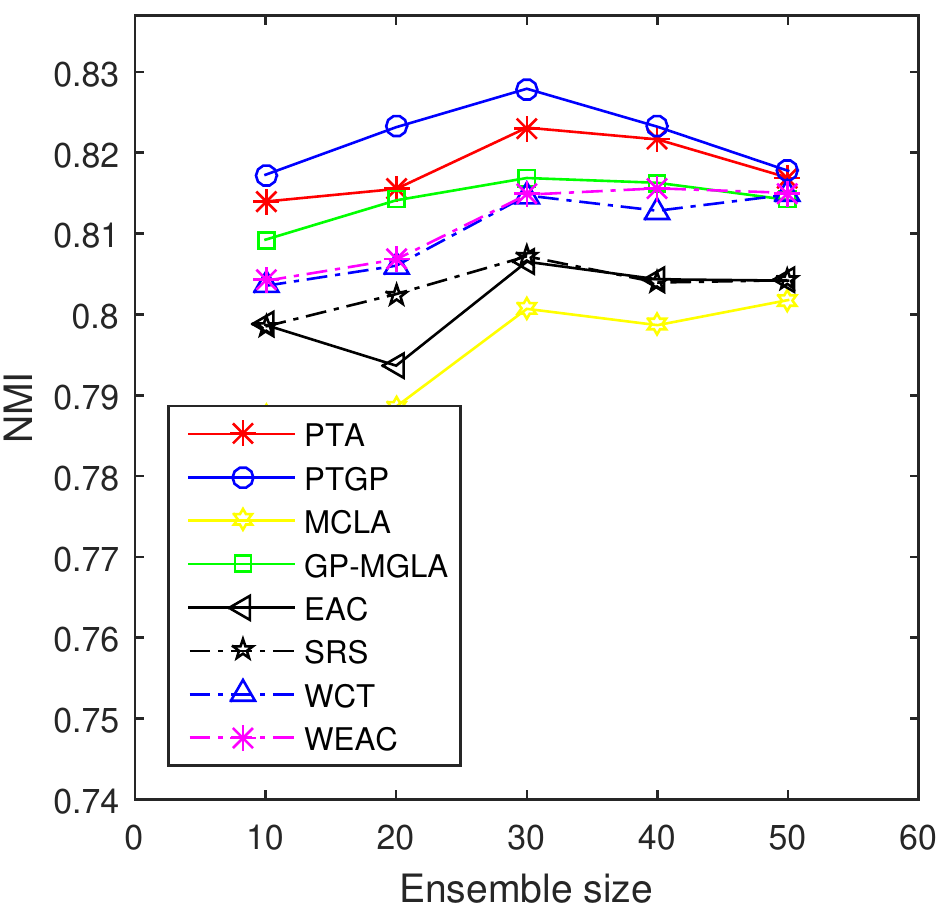}}}
{\subfigure[\emph{LS}]
{\includegraphics[width=0.395\columnwidth]{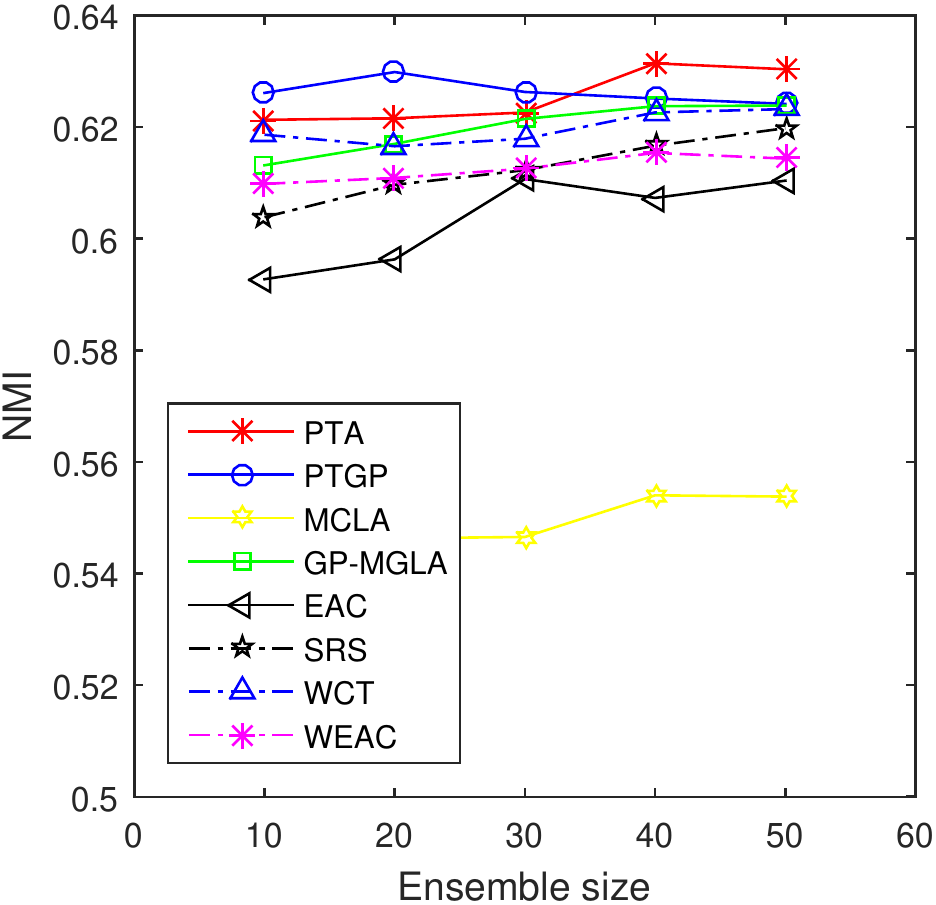}}}
{\subfigure[\emph{PD}]
{\includegraphics[width=0.395\columnwidth]{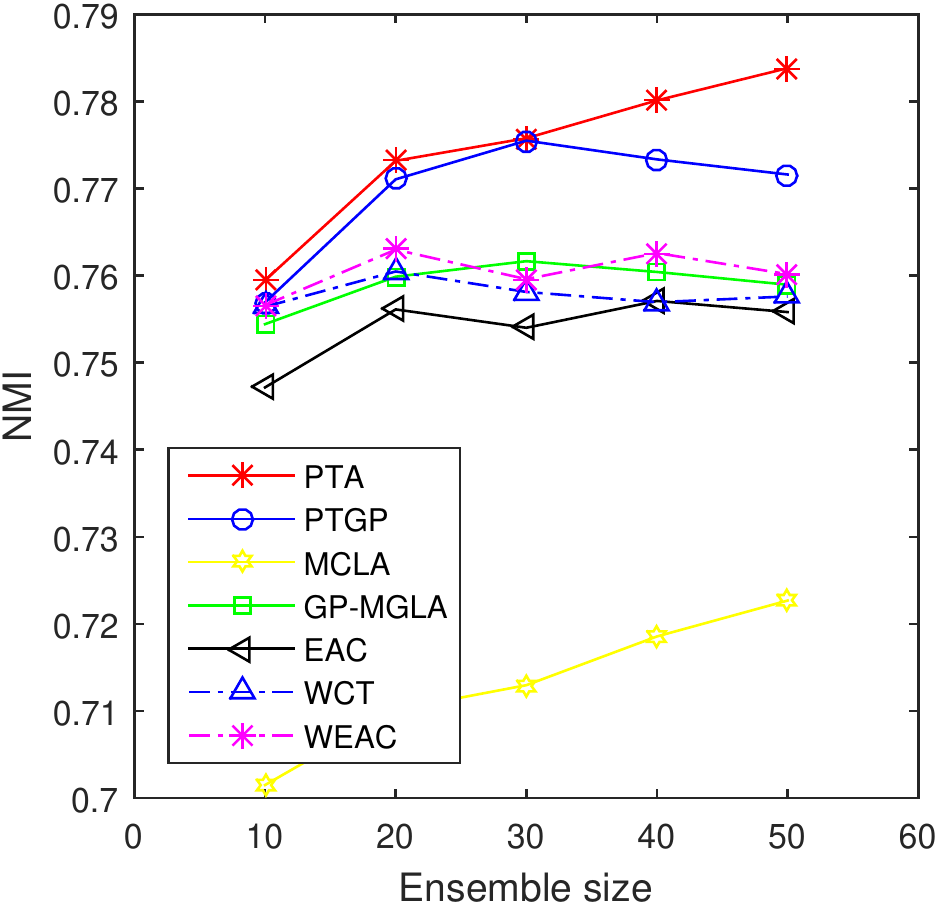}}}
{\subfigure[\emph{USPS}]
{\includegraphics[width=0.395\columnwidth]{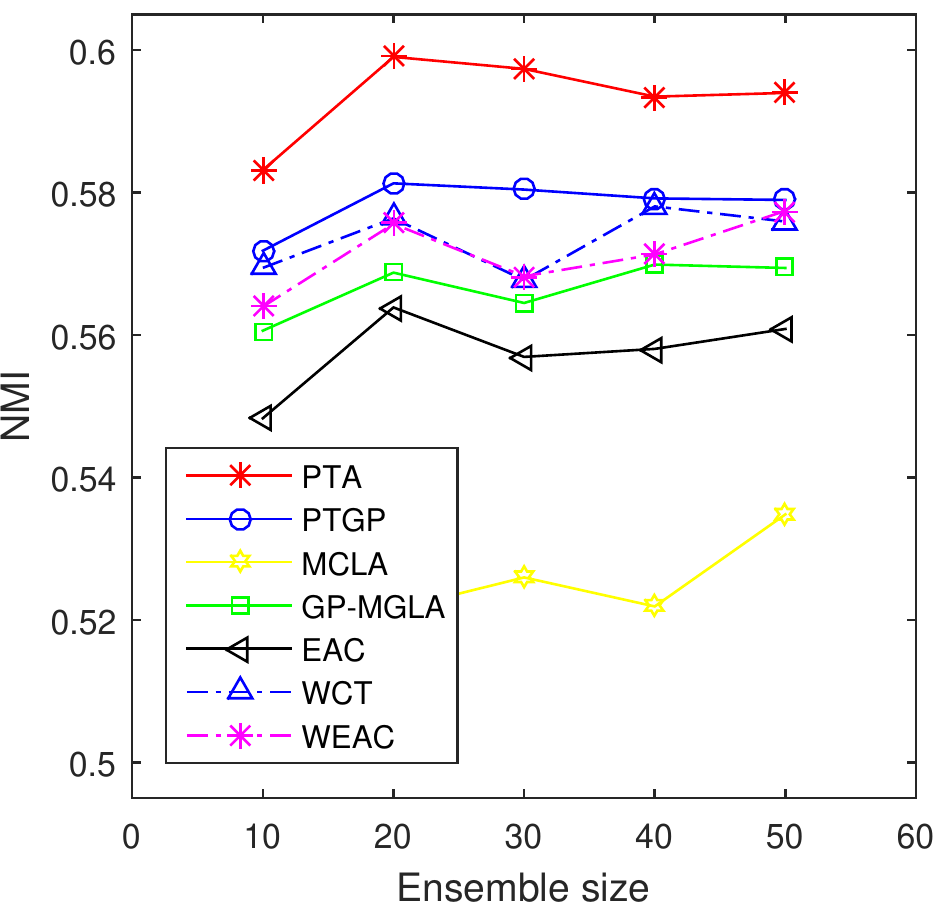}}}
{\subfigure[\emph{FC}]
{\includegraphics[width=0.395\columnwidth]{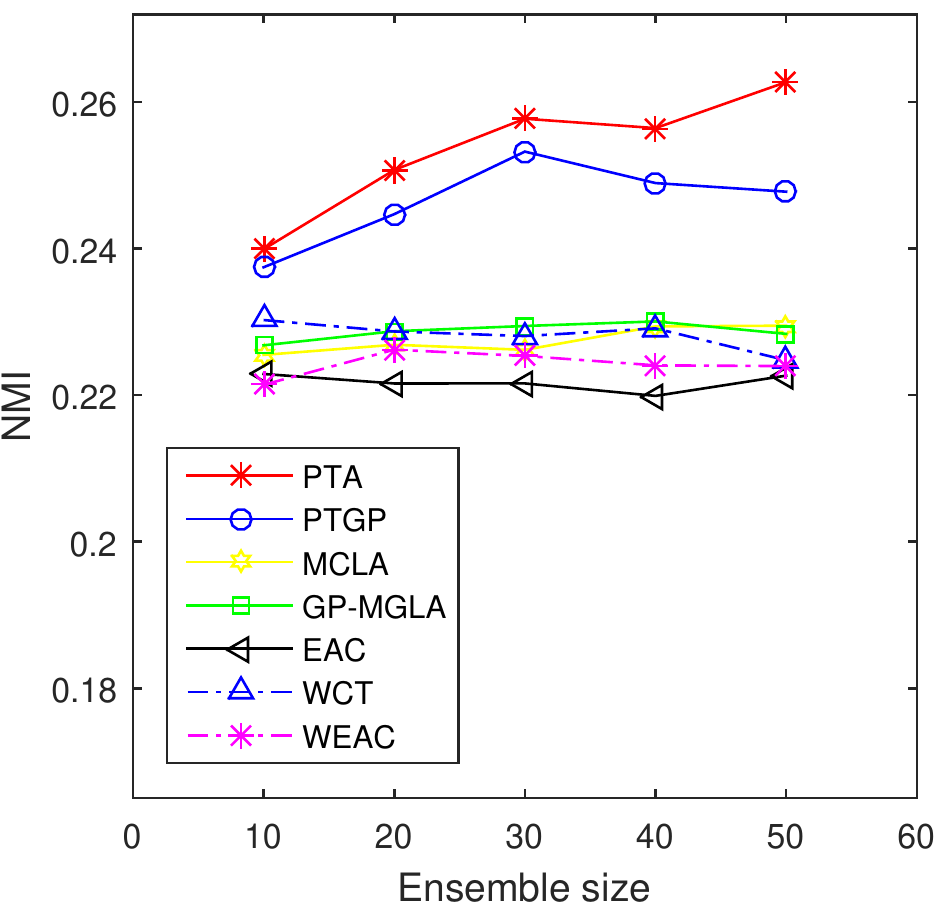}}}
{\subfigure[\emph{KDD99-10P}]
{\includegraphics[width=0.395\columnwidth]{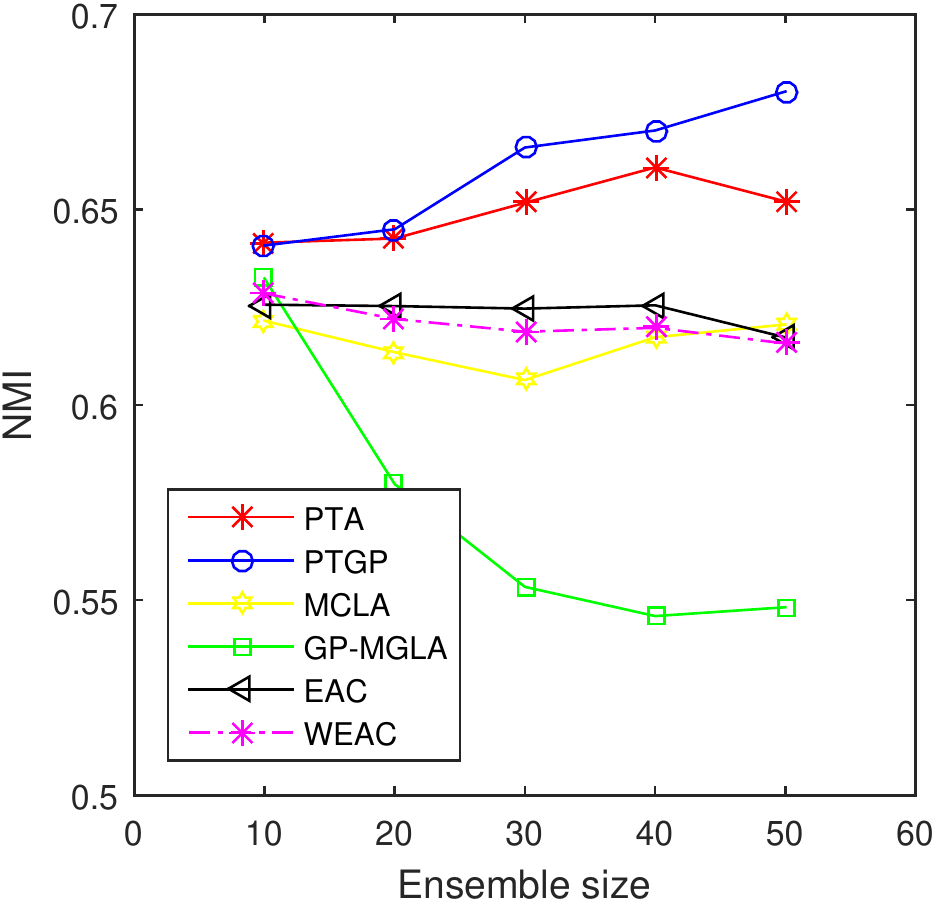}\label{fig:comp_Msize9}}}
{\subfigure[\emph{KDD99}]
{\includegraphics[width=0.395\columnwidth]{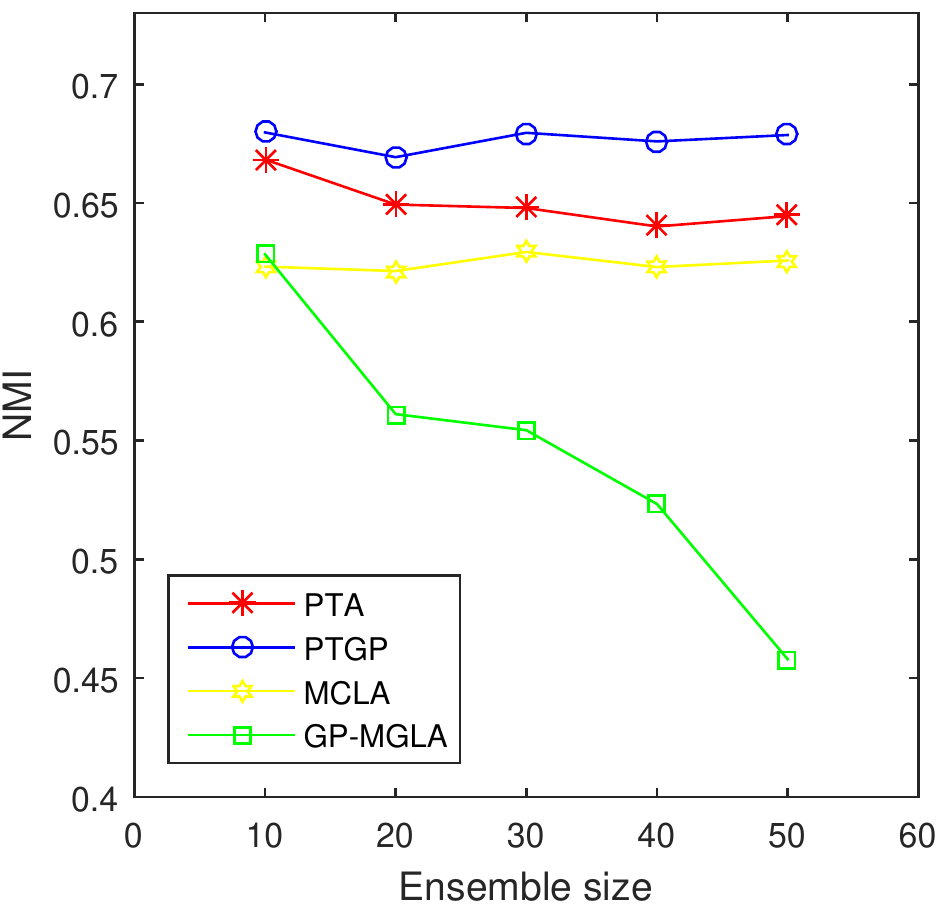}}}
\caption{The average performances over 20 runs by different approaches with varying ensemble sizes $M$.}
\label{fig:comp_Msize}
\end{center}
\end{figure*}

\subsection{Robustness to Ensemble Size $M$}
\label{sec:ensemble_size}

In this section, we further evaluate the robustness of the proposed methods with varying ensemble sizes $M$. For each ensemble size $M$, we run the PTA and PTGP methods and the baseline methods 20 times and report their average performances in Fig.~\ref{fig:comp_Msize}. Here, all pair-wise similarity based methods are associated with AL. The PTA and PTGP methods produce consistently good results with different ensemble sizes. As shown in Fig.~\ref{fig:comp_Msize}, the PTA and PTGP methods yield the best or nearly the best performance with varying $M$ for the benchmark datasets. Especially, for the \emph{MF}, \emph{IS}, \emph{PD}, \emph{FC}, \emph{KDD99-10P}, and \emph{KDD99} datasets, the PTA and PTGP methods exhibit significant advantages in the robustness to varying ensemble sizes over the baseline methods. Further, we illustrate the average performances of different approaches over nine datasets, KDD99 not included, in Fig.~\ref{fig:ensize9in1}, which is in fact the average of the first nine sub-figures in Fig.~\ref{fig:comp_Msize}, i.e., the sub-figures from Fig.~\ref{fig:comp_Msize1} to Fig.~\ref{fig:comp_Msize9}. Figure~\ref{fig:ensize9in1} provides an average view to compare our methods and the baseline methods across datasets, which demonstrates the advantage of our methods in the robustness to various datasets and ensemble sizes. In particular, the advantage of the proposed methods becomes even greater when the ensemble size gets larger, e.g., when the ensemble size goes beyond 20 (see Fig.~\ref{fig:ensize9in1}).

\subsection{Execution Time}
\label{sec:comp_time}

In this section, we evaluate the time performances of the proposed PTA and PTGP methods and the baseline methods with varying data sizes. The experiments are conducted on varying subsets of the \emph{KDD99} dataset. The sizes of the subsets range from $0$ to the full size $494,020$. The execution times of the proposed methods and the baseline methods w.r.t. varying data sizes are illustrated in Fig.~\ref{fig:time_complexity}. Note that the time lines of PTA and PTGP almost overlap with each other due to their very similar time performances. Besides that, the time lines of EAC and WEAC also nearly overlap with each other. As shown in Fig.~\ref{fig:time_complexity}, the proposed PTA and PTGP methods exhibit a significant advantage in efficiency over the baseline methods. Especially, the proposed PTGP and PTA methods consume 3.22 seconds and 3.70 seconds, respectively, to process the entire dataset of KDD99 which consists of nearly half a million objects, whereas eight out of the ten baseline methods are not even computationally feasible to process such large-scale datasets.

\section{Conclusion}
\label{sec:conclusion}

In this paper, we propose a novel ensemble clustering approach based on sparse graph representation and probability trajectory analysis. The microclusters are exploited as primitive objects to speedup the computation. We present the ENS strategy to identify uncertain links in a locally adaptive manner and construct a sparse graph with a small number of probably reliable links. It has been shown that the use of a small number of probably reliable links can lead to significantly better clusterings than using all graph links regardless of their reliability. To explore the global structure information in the ensemble, we utilize the random walks driven by a new transition probability matrix that considers the link weights and  the node sizes simultaneously. A novel and dense similarity measure termed PTS is derived from the sparse graph $K$-ENG by analyzing the probability trajectories of the random walkers. Based on PTS, we further propose two consensus functions, termed PTA and PTGP, respectively. Extensive experiments have been conducted on ten real-world datasets. The experimental results show that our approach significantly outperforms the state-of-the-art approaches in both clustering accuracy and efficiency.

\begin{figure}[!tb]
\label{fig:time_complexity}
\hskip 0.2in
\begin{center}
{\subfigure[]
{\includegraphics[width=0.5039\columnwidth]{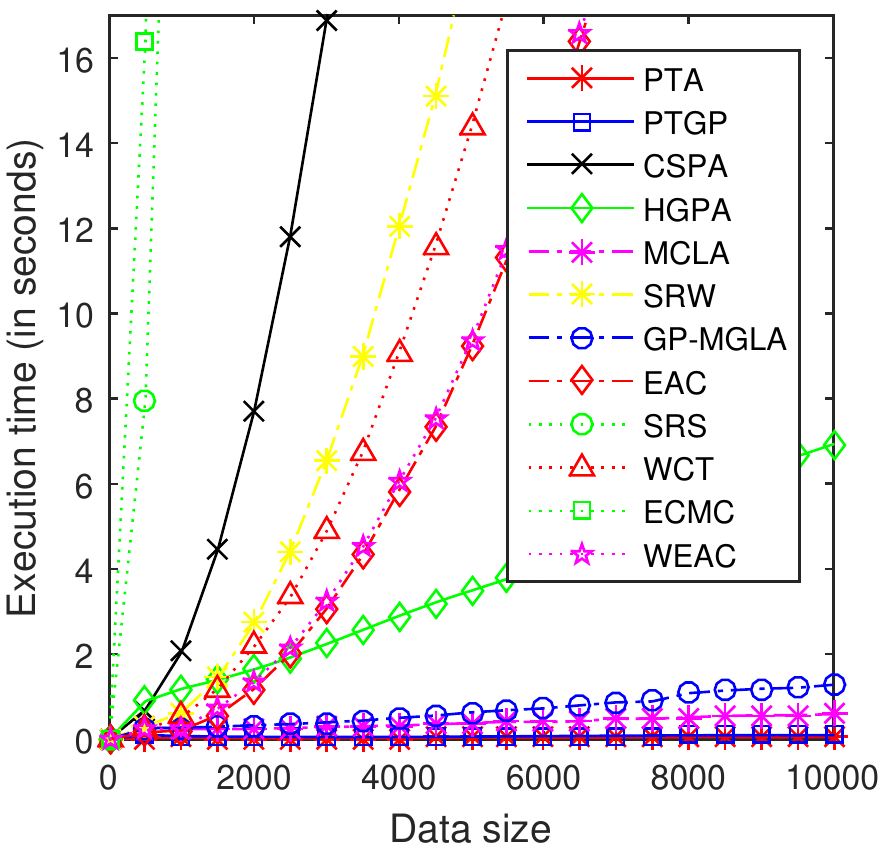}}}
{\subfigure[]
{\includegraphics[width=0.4864\columnwidth]{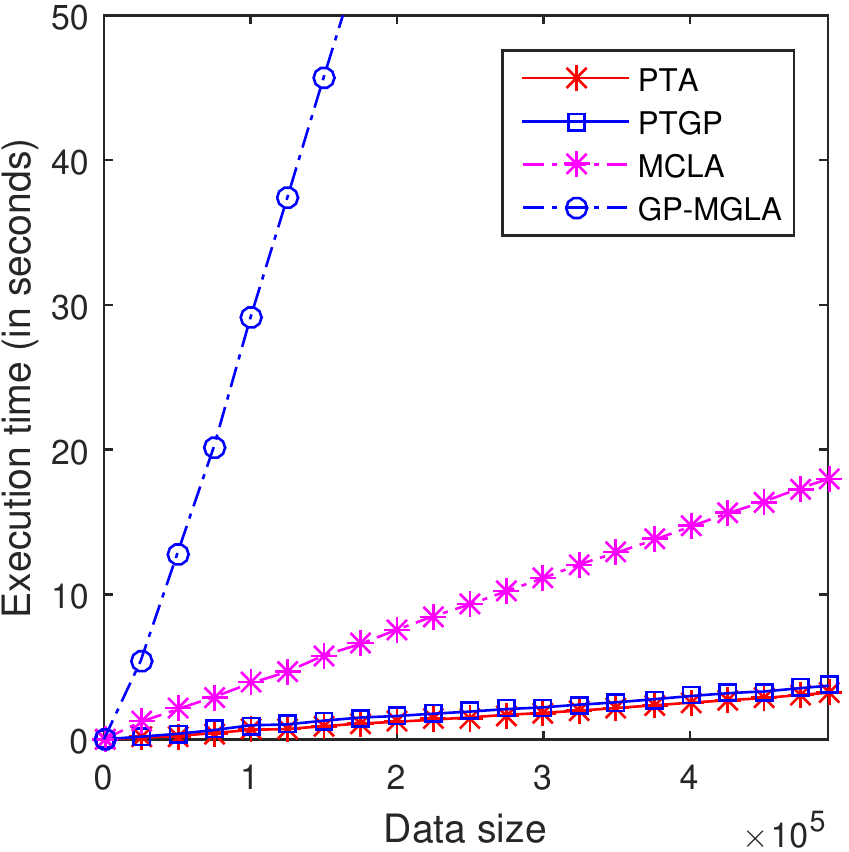}}}
\caption{Execution time of different ensemble clustering approaches as the data size varies (a) from 0 to 10,000 and (b) from 0 to 494,020. }
\label{fig:time_complexity}
\end{center}
\end{figure}

\ifCLASSOPTIONcompsoc
  \section*{Acknowledgments}
\else
  \section*{Acknowledgment}
\fi

The authors would like to thank the anonymous reviewers for their insightful comments and suggestions which helped enhance this paper significantly. This project was supported by NSFC (61173084 \& 61502543), National Science \& Technology Pillar Program (No. 2012BAK16B06), Guangdong Natural Science Funds for Distinguished Young Scholar, the GuangZhou Program (No. 201508010032), and the PhD Start-up Fund of Natural Science Foundation of Guangdong Province, China (No. 2014A030310180).

\ifCLASSOPTIONcaptionsoff
  \newpage
\fi

\bibliographystyle{IEEEtran}
\bibliography{tkde_2014_short}

\begin{IEEEbiography}[{\includegraphics[width=1in,height=1.25in,clip,keepaspectratio]{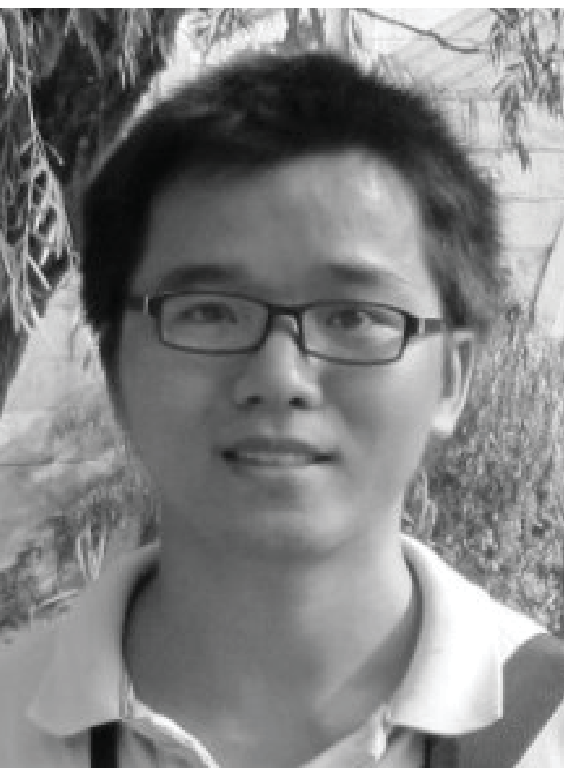}}]{Dong Huang}
received his B.S. degree in computer science in 2009 from South China University of Technology, China. He received his M.Sc. degree in computer science in 2011 and his Ph.D. degree in computer science in 2015, both from Sun Yat-sen University, China. He joined South China Agricultural University in 2015 as an Assistant Professor with College of Mathematics and Informatics. His research interests include data mining and pattern recognition. He is a member of the IEEE.
\end{IEEEbiography}

\begin{IEEEbiography}[{\includegraphics[width=1in,height=1.25in,clip,keepaspectratio]{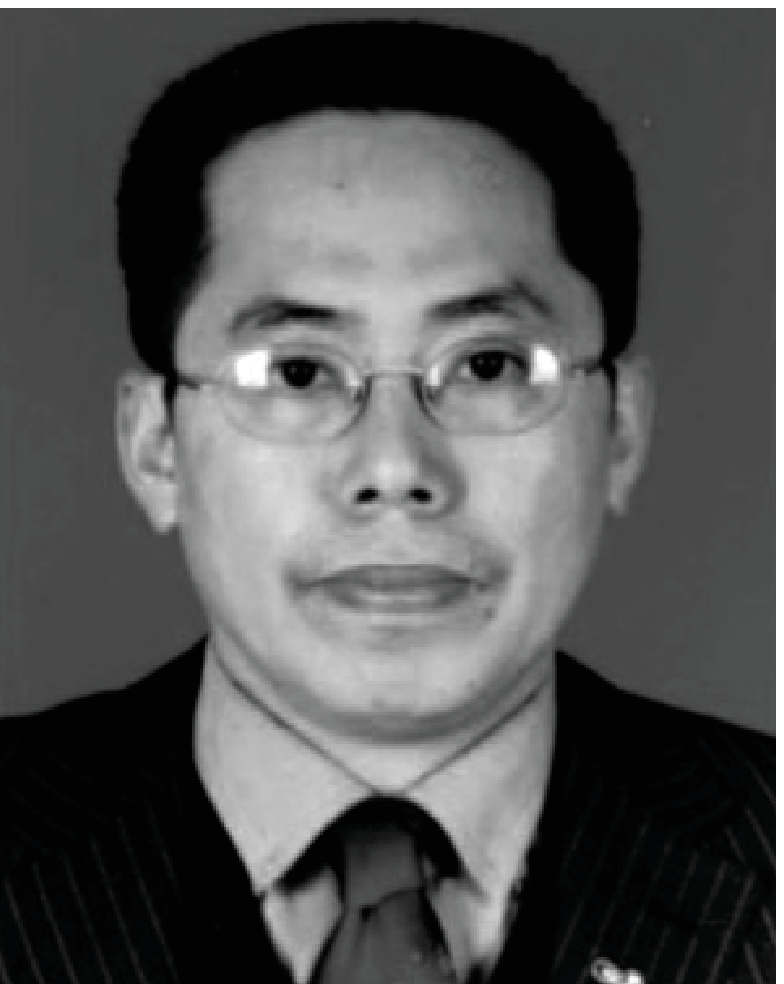}}]{Jian-Huang Lai}
received the M.Sc. degree in applied mathematics in 1989 and the Ph.D. degree in mathematics in 1999 from Sun Yat-sen University, China. He joined Sun Yat-sen University in 1989 as an Assistant Professor, where he is currently a Professor with the Department of Automation of School of Information Science and Technology, and Dean of School of Information Science and Technology. His current research interests include the areas of digital image processing, pattern recognition, multimedia communication, wavelet and its applications. He has published more than 100 scientific papers in the international journals and conferences on image processing and pattern recognition, such as IEEE TPAMI, IEEE TKDE, IEEE TNN, IEEE TIP, IEEE TSMC-B, Pattern Recognition, ICCV, CVPR, IJCAI, ICDM and SDM. Prof. Lai serves as a Standing Member of the Image and Graphics Association of China, and also serves as a Standing Director of the Image and Graphics Association of Guangdong. He is a senior member of the IEEE.
\end{IEEEbiography}

\begin{IEEEbiography}[{\includegraphics[width=1in,height=1.25in,clip,keepaspectratio]{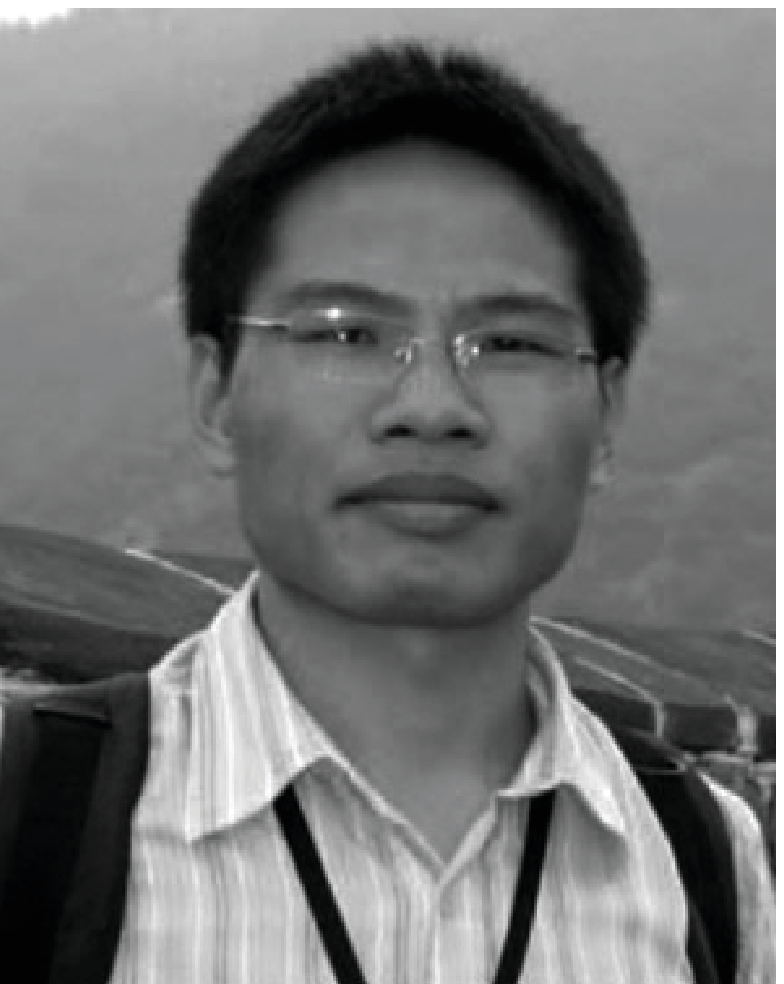}}]{Chang-Dong Wang}
received his Ph.D. degree in computer science in 2013 from Sun Yat-sen University, China. He is currently an assistant professor at School of Mobile Information Engineering, Sun Yat-sen University. His current research interests include machine learning and pattern recognition, especially focusing on data clustering and its applications. He has published over 30 scientific papers in international journals and conferences such as IEEE TPAMI, IEEE TKDE, IEEE TSMC-C, Pattern Recognition, Knowledge and Information System, Neurocomputing, ICDM and SDM. His ICDM 2010 paper won the Honorable Mention for Best Research Paper Awards. He won 2012 Microsoft Research Fellowship Nomination Award. He was awarded 2015 Chinese Association for Artificial Intelligence (CAAI) Outstanding Dissertation. He is a member of the IEEE.
\end{IEEEbiography}

\end{document}